\documentclass{article}
\usepackage{natbib}

\usepackage[top=30truemm,bottom=30truemm,left=30truemm,right=30truemm]{geometry}

\usepackage{times}
\usepackage{soul}
\usepackage{url}
\usepackage[hidelinks]{hyperref}
\usepackage[utf8]{inputenc}
\usepackage[small]{caption}
\usepackage{graphicx}
\usepackage{amsmath}
\usepackage{booktabs}
\usepackage{amssymb}
\usepackage{mathtools}
\usepackage{amsthm}
\usepackage{todonotes}
\usepackage{comment}

\newtheorem{theorem}{Theorem}
\newtheorem{corollary}{Corollary}
\newtheorem{lemma}{Lemma}
\newtheorem{prop}{Proposition}
\theoremstyle{definition}
\newtheorem{definition}{Definition}

\newcommand{\defeq}{\vcentcolon=}
\renewcommand{\S}{{\rm S}}
\newcommand{\D}{{\rm D}}
\newcommand{\U}{{\rm U}}
\newcommand{\E}{\mathbf{E}}
\newcommand{\R}{\mathbb{R}}
\newcommand{\bx}{\boldsymbol{x}}

\newcommand{\w}{\boldsymbol{w}}
\renewcommand{\hat}{\widehat}
\newcommand{\SU}{{\rm SU}}
\newcommand{\DU}{{\rm DU}}
\newcommand{\SD}{{\rm SD}}

\renewcommand{\cite}{\citep}

\usepackage{lipsum}

\usepackage{authblk}
\usepackage{subfig}

\title{Classification from Pairwise Similarities/Dissimilarities and Unlabeled Data via Empirical Risk Minimization}

\author[1]{Takuya Shimada}
\author[1,2]{Han Bao}
\author[1,2]{Issei Sato}
\author[2,1]{Masashi Sugiyama}
\affil[1]{The University of Tokyo}
\affil[2]{RIKEN}

\date{}

\begin{document}

\maketitle

\begin{abstract}
% Ver.3
Pairwise similarities and dissimilarities between data points might be easier to obtain
than fully labeled data in real-world classification problems, e.g., in privacy-aware situations.
To handle such pairwise information, an empirical risk minimization approach
has been proposed, giving an unbiased estimator of the classification
risk that can be computed only from pairwise similarities and unlabeled data.
However, this direction cannot handle pairwise dissimilarities so far.
% we may also obtain dissimilarities in practice.
% Thus, a method which can handle all of similarities/dissimilarities and unlabeled data is desirable.
On the other hand, semi-supervised clustering is one of the methods which can use both similarities and dissimilarities.
Nevertheless, they typically require strong geometrical assumptions on the
data distribution such as the manifold assumption, 
which may deteriorate the performance.
% which is often
% violated in practice and then their performance becomes unreliable.
In this paper, we derive an unbiased risk estimator which can handle all of similarities/dissimilarities and unlabeled data.
We theoretically establish estimation error bounds and experimentally
demonstrate the practical usefulness of our empirical risk minimization method.

\end{abstract}
\section{Introduction}
In supervised classification, we need a vast amount of labeled training data to train our classifiers.
However, it is often not easy to obtain labels due to 
high labeling costs~\cite{chapelle2010semi},
privacy concern~\cite{warner1965randomized},
social bias~\cite{nederhof1985methods},
and difficulty to label data.
For such reasons, there is a situation in real-world classification problems,
where pairwise similarities (i.e., pairs of samples in the same class)
and pairwise dissimilarities (i.e., pairs of samples in different classes) might be easier to collect than fully labeled data.
For example, in the task of protein function prediction~\cite{Klein2002FromIC},
the knowledge about similarities/dissimilarities can be obtained as additional supervision,
% are used for augmenting genome sequence,
which can be found by experimental means.
To handle such pairwise information, similar-unlabeled (SU) classification~\cite{bao2018classification} has been proposed, where the classification risk is estimated in an unbiased fashion from only similar pairs and unlabeled data.
Although they assumed that only similar pairs and unlabeled data are available,
we may also obtain dissimilar pairs in practice.
In this case, a method which can handle all of
similarities/dissimilarities and unlabeled data is desirable.
%we should consider a case where dissimilar pairs are also available.
%This is because dissimilar pairs are often spontaneously obtained in the process of labeling similarities;
%that is, pairs labeled not to be similar are usually dissimilar pairs.
%Therefore, 
%a method which can handle all of similarities/dissimilarities and unlabeled data is desirable.

Semi-supervised clustering~\cite{wagstaff2001constrained} is one of the methods 
that can handle both similar and dissimilar pairs,
where must-link pairs (i.e., similar pairs) 
and cannot-link pairs (i.e., dissimilar pairs) are used to obtain meaningful clusters.
Existing work provides useful semi-supervised clustering methods based on the ideas that 
(i) must/cannot-links are treated as constraints~\cite{basu2002semi,wagstaff2001constrained,li2009constrained,hu2008maximum}, 
(ii) clustering is performed with metrics learned by semi-supervised metric learning~\cite{xing2003distance,bilenko2004integrating,weinberger2009distance,davis2007information,niu2012information},
and (iii) missing links are predicted by matrix completion~\cite{yi2013semi,chiang2015matrix}.
However, there is a gap between the motivation of clustering and classification algorithms, 
so applying semi-supervised clustering to classification might cause a problem.
For example, most of the semi-supervised clustering methods rely on geometrical or margin-based assumptions such as the cluster assumption and manifold assumption~\cite{chapelle2010semi},
which heavily depend on the structure of datasets.
Therefore, the range of applications of semi-supervised clustering can be restricted.
In addition, the objective of semi-supervised clustering is not basically the minimization of the classification risk, 
which may perform suboptimally in terms of classification accuracy.

In this paper, 
we propose similar-dissimilar-unlabeled (SDU) classification, 
where we can utilize all of pairwise similarities/dissimilarities and unlabeled data 
for unbiased estimation of the classification risk.
Similarly to SU classification, our method does not require geometrical assumptions on the data distribution and directly minimizes the classification risk.
As the first step to construct our SDU classification, we propose dissimilar-unlabeled (DU) classification and similar-dissimilar classification (SD), where only dissimilar and unlabeled data or similar and dissimilar data are required.
Then, we combine the risks of SU, DU, and SD classification linearly
in a similar manner to positive-negative-unlabeled (PNU) classification~\cite{sakai2017semi}.
One important question is which combination of these three risks is the best.
To answer this question, we establish estimation error bounds for each algorithm
and find that SD and DU classification are likely to outperform SU classification in terms of generalization. 
Therefore, we claim that the combination of SD and DU classification is the most promising approach for SDU classification.
Through experiments, we demonstrate the practical usefulness of our proposed method.

Our contributions can be summarized as follows.
\begin{itemize}
    \item We extend SU classification to DU and SD classification
    and propose SDU classification by combination of those algorithms 
    (Sec.\,\ref{sec:propose}).
    \item We establish estimation error bounds for each algorithm 
    and confirm also that unlabeled data helps to estimate the classification risk. 
    (Sec.\,\ref{sec:bounds_sudusd} and Sec.\,\ref{sec:bounds_sdu}).
    \item From the comparison of the estimation error bounds, 
    we find that SD classification and DU classification are likely to outperform SU classification,
    and provide an insight that the combination of SD and DU classification is the most promising for SDU classification 
    (Sec.\,\ref{sec:theo_sudusd}).
\end{itemize}
\section{Preliminary}
In this section, we first introduce our problem setting
and data generation process of similar pairs, dissimilar pairs, and unlabeled data.
Then we review the formulation of the existing SU classification algorithm.

\subsection{Problem Setting}
Let $\mathcal{X} \subset \mathbb{R}^d$ and $\mathcal{Y}=\{+1, -1\}$ 
be a $d$-dimensional example space and binary label space, respectively.
Suppose that each labeled example $(\boldsymbol{x},y) \in \mathcal{X} \times \mathcal{Y}$ is generated from the joint probability with density $p(\boldsymbol{x}, y)$ independently.
For simplicity, let $\pi_+$ and $\pi_-$ be class priors $p(y=+1)$ and $p(y=-1)$, which satisfy the condition $\pi_+ + \pi_-=1$, and $p_+(\bx)$ and $p_-(\bx)$ be class conditional densities $p(\bx | y=+1)$ and $p(\bx | y=-1)$.

The standard goal of supervised binary classification is to obtain a classifier $f: \mathcal{X} \rightarrow \mathbb{R}$ which minimizes the classification risk defined by

\begin{equation}
\label{eq:exp_pn_risk}
    R(f) \defeq \E_{(X,Y) \sim p(\boldsymbol{x},y)} \left[ \ell(f(X),Y) \right],
\end{equation}
where $\E_{(X,Y) \sim p(\boldsymbol{x},y)} \left[ \cdot \right]$ denotes the expected value over joint density $p(\boldsymbol{x},y)$ 
and $\ell : \mathbb{R} \times \mathcal{Y} \rightarrow \mathbb{R}_{+}$ is a loss function.

\subsection{Generation Process of Training Data}
We describe the data generation process of pairwise similar and dissimilar data and unlabeled data.
%\subsubsection{Similar and Dissimilar Pairs}
We assume that similar and dissimilar pairs are generated from pairwise distributions independently. 
We denote the event that two samples $(\boldsymbol{x}, y)$ and $(\boldsymbol{x}', y')$ have the same class label (i.e., $y=y'$) by $s=+1$, and otherwise by $s=-1$.
Then, similar and dissimilar pairs are generated from an underlying joint density $p(\boldsymbol{x}, \boldsymbol{x}', s)$ as follows:

\begin{align}
\label{eq:gen_process_sd}
    & \mathfrak{D_{\SD}} \defeq \{ (\boldsymbol{x}_{{\SD}. i}, \boldsymbol{x}'_{{\SD}, i}, s_i) \}_{i=1}^{n_{\SD}} \sim p(\boldsymbol{x}, \boldsymbol{x}', s),
\end{align}
where
\begin{align}
    & p(\boldsymbol{x}, \boldsymbol{x}', s=+1) = p(s=+1) p(\boldsymbol{x}, \boldsymbol{x}' | s=+1) = p(y=y') p(\boldsymbol{x}, \boldsymbol{x}' | y=y'), \\
    & p(\boldsymbol{x}, \boldsymbol{x}', s=-1) = p(s=-1) p(\boldsymbol{x}, \boldsymbol{x}' | s=-1) = p(y \neq y') p(\boldsymbol{x}, \boldsymbol{x}' | y \neq y'). 
\end{align}
Here, $n_\SD$ pairs in $\mathfrak{D}_\SD$ can be decomposed into 
$n_\S$ similar pairs and $n_\D$ dissimilar pairs based on the variable $s$.
%Note that $\mathfrak{D}_{\SD}$ can be decomposed into pairwise similar data $\mathfrak{D}_\S$ and dissimilar data $\mathfrak{D}_\D$ based on the variable $s$.
\begin{align}
    & \mathfrak{D}_\S \defeq
     \{ (\boldsymbol{x}_{{\S}. i}, \boldsymbol{x}'_{{\S}, i})\}_{i=1}^{n_{\S}} 
     = \left\{ (\boldsymbol{x}, \boldsymbol{x}') \mid \left(\boldsymbol{x}, \boldsymbol{x}', s=+1 \right) \in \mathfrak{D}_\SD \right\}, \\
    & \mathfrak{D}_\D \defeq
    \{ (\boldsymbol{x}_{{\D}. i}, \boldsymbol{x}'_{{\D}, i})\}_{i=1}^{n_{\D}} 
    = \left\{ (\boldsymbol{x}, \boldsymbol{x}') \mid \left(\boldsymbol{x}, \boldsymbol{x}', s=-1 \right) \in \mathfrak{D}_\SD \right\}.
\end{align}
For convenience, we introduce notations representing the similar and dissimilar proportions and conditional densities.
\begin{align}
    & \pi_\S \defeq p(y=y'), \\
    & \pi_\D \defeq p(y \neq y'), \\
    & p_\S(\boldsymbol{x},\boldsymbol{x}') \defeq p(\boldsymbol{x}, \boldsymbol{x}' | y=y'), \\
    & p_\D(\boldsymbol{x},\boldsymbol{x}') \defeq p(\boldsymbol{x}, \boldsymbol{x}' | y \neq y').
\end{align}
Then, we can consider the generation process of similar and dissimilar pairs as
\begin{align*}
    & \mathfrak{D}_\S \sim p_\S(\boldsymbol{x}, \boldsymbol{x}'), \\
    & \mathfrak{D}_\D \sim p_\D(\boldsymbol{x}, \boldsymbol{x}').
\end{align*}
Note that we assume each sample in a pair is generated independently, namely,
$(\bx,y), (\boldsymbol{x}', y') \sim p(\bx, y)$.
Thus, we have
\begin{align*}
    & \pi_\S = p(y=+1)p(y'=+1) + p(y=-1)p(y'=-1) 
     = \pi_+^2 + \pi_-^2, \\
    & \pi_\D = p(y=+1)p(y'=-1) + p(y=-1)p(y'=+1) 
     = 2 \pi_+ \pi_-, \\
    & p_\S(\bx, \bx') = \frac{\pi_+^2}{\pi_\S} p_+(\bx)p_+(\bx')
                + \frac{\pi_-^2}{\pi_\S} p_-(\bx)p_-(\bx'), \\
    & p_\D(\bx, \bx') = \frac{1}{2} p_+(\bx)p_-(\bx')
    + \frac{1}{2} p_-(\bx)p_+(\bx').
\end{align*}
%\subsubsection{Unlabeled Data}
We assume unlabeled samples are generated as follows:
\begin{align}
    \mathfrak{D_\U} \defeq \{\boldsymbol{x}_{\U,i}\}_{i=1}^{n_\U} \sim p_\U(\boldsymbol{x}) = \pi_+ p_+(\boldsymbol{x}) + \pi_- p_-(\boldsymbol{x}).
\end{align}

\subsection{SU classification}
In \cite{bao2018classification}, SU classification was proposed,
where the classification risk is estimated in an unbiased fashion only from similar pairs and unlabeled data.
\begin{prop}[Theorem~1 in \cite{bao2018classification}]
\label{theo:su}
The classification risk in Eq.\,\eqref{eq:exp_pn_risk} can be equivalently represented as
\begin{equation}
\begin{split}
    R_{\SU}(f) = & \pi_\S \E_{(X,X') \sim p_\S(\boldsymbol{x}, \boldsymbol{x}')}
    \left[
            \frac{\widetilde{\mathcal{L}}(f(X)) + \widetilde{\mathcal{L}}(f(X'))}{2}
    \right] \\
    & + \E_{X \sim p_\U(\boldsymbol{x})}
    \left[
            \mathcal{L}(f(X), -1)
    \right],
\end{split}
\end{equation}
where
\begin{align}
    & \mathcal{L}(z, t) \defeq \frac{\pi_{+}}{\pi_{+} - \pi_{-}} \ell(z, t) - \frac{\pi_{-}}{\pi_{+} - \pi_{-}} \ell(z, -t), \\
    & \widetilde{\mathcal{L}} (z) \defeq \frac{1}{\pi_+ - \pi_-} \ell(z,+1) - \frac{1}{\pi_+ - \pi_-} \ell(z,-1).
\end{align}
\end{prop}
We can train a classifier by minimizing the empirical version of ${R}_\SU$, $\widehat{R}_\SU$,
from $(\mathfrak{D}_\S, \mathfrak{D}_\U)$.
%through the empirical risk , $\widehat{R}_\SU$ obtained by $(\mathfrak{D}_\S, \mathfrak{D}_\U)$. Our goal in this paper is to extend their SU classification 
%so that we can also handle dissimilar pairs.
\section{Proposed Method}

\label{sec:propose}
In this section, we propose SDU classification, 
where the classification risk is estimated from similar and dissimilar pairs and unlabeled data.
For the preparation to construct SDU classification,
we extend SU classification to DU and SD classification first.

\subsection{DU and SD classification}
\label{sec:dusd}
As well as SU classification, the classification risk can be estimated only from
dissimilar pairs and unlabeled data (DU),
or similar pairs and dissimilar pairs (SD) as follows.

\begin{theorem}
\label{theo:dusd_risk}
The classification risk in Eq.\,\eqref{eq:exp_pn_risk} can be equivalently represented as
\begin{equation}
\begin{split}
    R_{\DU}(f) = & \pi_\D \E_{(X,X') \sim p_\D(\boldsymbol{x}, \boldsymbol{x}')}
    \left[
            - \frac{\widetilde{\mathcal{L}}(f(X)) + \widetilde{\mathcal{L}}(f(X'))}{2}
    \right] \\
    & + \E_{X \sim p_\U(\boldsymbol{x})}
    \left[
            \mathcal{L}(f(X), +1)
    \right],
\end{split}
\end{equation}
\begin{equation}
\begin{split}
R_{\rm SD}(f) 
    & = \pi_\S \E_{(X,X') \sim p_\S(\boldsymbol{x}, \boldsymbol{x}')}
    \left[
            \frac{\mathcal{L}(f(X), +1) + \mathcal{L}(f(X'), +1)}{2}
    \right] \\
    & \quad + \pi_\D \E_{(X,X') \sim p_\D (\boldsymbol{x}, \boldsymbol{x}')}
    \left[
            \frac{\mathcal{L}(f(X), -1) + \mathcal{L}(f(X'), -1)}{2}
    \right],
\end{split}
\end{equation}
where $\mathcal{L}(z, t)$ and $\widetilde{\mathcal{L}} (z)$ are defined in Theorem~\ref{theo:su}.\footnote{Due to limited space, the proofs of theorems are shown in Appendix \ref{sec_a:proofs}.}
\end{theorem}
We can train a classifier by minimizing the empirical version of $R_\DU$ or $R_\SD$, $\widehat{R}_\DU$ or $\widehat{R}_\SD$,
obtained from $(\mathfrak{D}_\D, \mathfrak{D}_\U)$ or $(\mathfrak{D}_\S, \mathfrak{D}_\D)$.
We call the training with these risks DU classification and SD classification, respectively.

\subsection{SDU classification}
We propose SDU classification by combining SU, DU, and SD classification.
The main idea of our method is to combine risks obtained from SU, DU and SD data in a similar manner to Positive-Negative-Unlabeled (PNU) classification~\cite{sakai2017semi}.

With a positive real value $\gamma \in [0,1]$,
we define the following three representations of the classification risk.
\begin{align}
\label{eq:risk_SDSU}
    & R_{\rm SDSU}^{\gamma}(f) \defeq (1-\gamma) R_{\rm SD}(f) + \gamma R_{\rm SU}(f), \\
\label{eq:risk_SDDU}
    & R_{\rm SDDU}^{\gamma}(f) \defeq (1-\gamma) R_{\rm SD}(f) + \gamma R_{\rm DU}(f), \\
\label{eq:risk_SUDU}
    & R_{\rm SUDU}^{\gamma}(f) \defeq (1-\gamma) R_{\rm SU}(f) + \gamma R_{\rm DU}(f).
\end{align}
We call the training with these risks SDSU classification, SDDU classification, and SUDU classification, respectively.
Here, one spontaneous question is that which one is the most promising algorithm.
We can claim the combination of SD and DU risks is the most promising from the point of view of estimation error bounds.
We discuss the details in Sec.\,\ref{sec:theo_sudusd}.

\subsection{Practical Implementation}
We investigate the objective function when using a linear classifier 
$f(\bx) = {\w}^{\top} \boldsymbol{\phi}(\bx) + b$, 
where $\w \in \R^k$ and $b \in \R$ are parameters and $\boldsymbol{\phi}: \R^d \rightarrow \R^k$ is a mapping function.
For simplicity, we consider the following generalized optimization problem,
With positive real values  $ \boldsymbol{\gamma} = \{ \gamma_1,\gamma_2,\gamma_3 \} $ which satisfy the condition $\gamma_1 + \gamma_2 + \gamma_3 = 1$, we denote our optimization problem by
% With the empirical risk $\widehat{R}_{\rm SDU}$, we denote our optimization problem by
% Our optimization problem is denoted by
\begin{equation}
\label{eq:min_j}
\min_{\w} \hat{J}^{\boldsymbol{\gamma}}(\w),
\end{equation}
where
\begin{align}
\label{eq:obj}
 & \hat{J}^{\boldsymbol{\gamma}}(\w) = \hat{R}_{\rm SDU}^{\boldsymbol{\gamma}}(\w) + \frac{\lambda}{2} \| \w \|^2, \\
 \label{eq:emp_risk_sdu}
 & \hat{R}_{\rm SDU}^{\boldsymbol{\gamma}} (\w)
 = \gamma_1 \hat{R}_{\rm SU}(\w) + \gamma_2 \hat{R}_{\rm DU}(\w) + \gamma_3 \hat{R}_{\rm SD}(\w).
\end{align}
Here, $\lambda > 0$ is a parameter of L2 regularization.
When $\gamma_1 = 0$, $\gamma_2=0$ and $\gamma_3=0$, 
this optimization corresponds to SDDU, SDSU, SUDU classification, respectively.

From now on, we assume the loss function is a margin loss function.
As defined in \cite{mohri2018foundations}, we call $\ell$ a margin loss function if there exists $\psi: \R \rightarrow \R_+$ such that $\ell(z,t)=\psi(tz).$
In general, the optimization problem in Eq.\,\eqref{eq:min_j} is not a convex problem.
However, if we choose $\ell$ satisfies the following property, the optimization problem becomes convex. 
\begin{theorem}
\label{theo:conv}
Suppose that the loss function $\ell(z,t)$ is a convex margin loss,
twice differentiable in z almost everywhere (for every fixed $t\in\{\pm 1\}$) and satisfies following condition.
\begin{equation}
\label{eq:loss_assumption}
    \ell(z, +1) - \ell(z, -1) = -z.
\end{equation}
Then, the optimization problem in Eq.\,(\ref{eq:min_j}) is a convex problem.
\end{theorem}
Next, we consider the case of squared loss and double hinge loss, which satisfy the condition in Eq.\,\eqref{eq:loss_assumption}.

\subsubsection{Squared Loss}
Suppose that we use squared loss defined by
\begin{equation}
    \label{eq:loss_sq}
    \ell_{\rm SQ}(z, t) = \frac{1}{4}(tz-1)^2.
\end{equation}
Then, the objective function in $\widehat{J}^{\boldsymbol{\gamma}}(\w)$ can be written as
\begin{equation}
    \begin{split}
    \widehat{J}^{\boldsymbol{\gamma}}(\w)
    & = \frac{1}{4} \w^\top \left\{\gamma_3 \left( \frac{\pi_\S}{2n_\S}X_\S^\top X_\S + \frac{\pi_\D}{2n_\D}X_\D^\top X_\D \right) \right. 
    \left. + \frac{\gamma_1 + \gamma_2}{n_\U} X_\U^\top X_\U + 2 \lambda I \right\} \w \\
  & +\frac{1}{\pi_+ - \pi_-}\left\{ -\frac{\pi_\S}{2 n_\S}\left( \gamma_1 + \frac{\gamma_3}{2} \right)  X_\S^\top \mathbf{1} \right.
   \left. + \frac{\pi_\D}{2 n_\D}\left( \gamma_2 + \frac{\gamma_3}{2} \right) X_\D^\top \mathbf{1}
  + \frac{1}{2 n_\U}\left( \gamma_1 - \gamma_2 \right) X_\U^\top \mathbf{1} 
  \right\} \w \\
  & +\text{const.},
\end{split}
\end{equation}
\begin{comment}
\begin{equation}
\begin{split}
    & \resizebox{1.1 \linewidth}{!}{$\frac{1}{4} \w^\top \left\{\frac{\gamma_3}{2} \left( \frac{\pi_\S}{n_\S}X_\S^\top X_\S + \frac{\pi_\D}{n_\D}X_\D^\top X_\D \right) 
     + \frac{\gamma_1 + \gamma_2}{n_\U} X_\U^\top X_\U + 2 \lambda I \right\} \w $}\\
   & \resizebox{1. \linewidth}{!}{$+\frac{1}{\pi_+ - \pi_-} \left\{ -\frac{\pi_\S}{2 n_\S}\left( \gamma_1 + \frac{\gamma_3}{2} \right)  X_\S^\top \mathbf{1} + \frac{\pi_\D}{2 n_\D}\left( \gamma_2 + \frac{\gamma_3}{2} \right) X_\D^\top \mathbf{1} \right.$} \\ 
   & \resizebox{0.5 \linewidth}{!}{$\left. + \frac{1}{2 n_\U}\left( \gamma_1 - \gamma_2 \right) X_\U^\top \mathbf{1} \right\} \w + \text{const}$},
\end{split}
\end{equation}
\begin{equation}

\end{equation}
\end{comment}
where 

\begin{align*}
    X_\S &\defeq [\boldsymbol{\phi}(\bx_{\S,1}),\boldsymbol{\phi}(\bx'_{\S,1}),\dots, \boldsymbol{\phi}(\bx_{\S,n_\S}),\boldsymbol{\phi}(\bx'_{\S,n_\S})]^\top, \\
    X_\D &\defeq [\boldsymbol{\phi}(\bx_{\D,1}),\boldsymbol{\phi}(\bx'_{\D,1}),\dots,  \boldsymbol{\phi}(\bx_{\D,n_\D}),\phi(\bx'_{\D,n_\D})]^\top, \\
    X_\U &\defeq [\boldsymbol{\phi}(\bx_{\U,1}),\dots, \boldsymbol{\phi}(\bx_{\U, n_\U})]^\top.
\end{align*}
We denote $\mathbf{1}$ as the vector whose elements are all ones and $I$ as the identity matrix. 
Since this function has a quadratic form with respect to $\w$,
the solution of this minimization problem can be obtained analytically.

\begin{comment}
    \begin{equation}
    \begin{split}
    \hat{\w} = & \frac{1}{\pi_+ - \pi_-} 
     \left\{\gamma_3 \left( \frac{\pi_\S}{2n_\S}X_\S^\top X_\S + \frac{\pi_\D}{2n_\D}X_\D^\top X_\D \right) \right. \\
      & \left. + \frac{\gamma_1 + \gamma_2}{n_\U} X_\U^\top X_\U + 2 \lambda I \right\}^{-1} \\
    &\left\{ \frac{\pi_\S}{n_\S}\left( \gamma_1 + \frac{\gamma_3}{2} \right)  X_\S^\top \mathbf{1} + \frac{\pi_\D}{n_\D}\left( \gamma_2 + \frac{\gamma_3}{2} \right) X_\D^\top \mathbf{1} \right. \\
    & \left.  + \frac{1}{n_\U}\left( \gamma_1 - \gamma_2 \right) X_\U^\top \mathbf{1}
      \right\}.
    \end{split}
    \end{equation}
\end{comment}

\subsubsection{Double Hinge Loss}
Standard hinge loss $\ell_{\rm H}(z,t) = \max(0, -tz)$ does not satisfy condition in Eq.\,\eqref{eq:loss_assumption}. 
As an alternative, double hinge loss 
$\ell_{\rm DH}(z, t) = \max (-tz, \max (0, \frac{1}{2}-\frac{1}{2}tz)$
is proposed by \cite{du2015convex}.
When we use $\ell_{\rm DH}$,
we can solve the optimization problem in Eq.\,\eqref{eq:min_j}
by quadratic programming.\footnote{The details is described in Appendix \ref{sec_a:opt}.}

\subsection{Class Prior Estimation}
Although we have to know the class prior $\pi_+$ before training
for calculation of empirical risks $\widehat{R}_\SD$, $\widehat{R}_\SU$, and $\widehat{R}_\DU$,
$\pi_+$ can be estimated from the number of similar pairs $n_\S$ and the number of dissimilar pairs $n_\D$.
First, $\pi_+$ and $\pi_\S$ has following relationship.
\begin{equation}
\label{eq:cpe}
    \pi_+ = 
    \begin{cases}
    \frac{1 + \sqrt{2\pi_\S - 1}}{2} & (\pi_+ \geq 0.5), \\
    \frac{1 - \sqrt{2\pi_\S - 1}}{2} & (\pi_+ < 0.5),
    \end{cases}
\end{equation}
The above equality is obtained from 
$2\pi_\S - 1 = \pi_\S - \pi_\D = (\pi_+ - \pi_-) = (2\pi_+ - 1)^2$.
Note that $\widehat{\pi}_\S = {n_\S} / ({n_\S + n_\D})$ is an unbiased estimator of $\pi_\S$.
Thus, $\pi_+$ can be estimated by plugging $\widehat{\pi}_\S$ into Eq.\,\eqref{eq:cpe}.
%From $\widehat{\pi}_\S = {n_\S} / ({n_\S + n_\D})$ estimated 
%from the number of similar and dissimilar pairs, 
%i.e., $n_\S$ and $n_\D$,
%we obtain estimated $\widehat{\pi}_+$.
% Even if we don't know true class prior in practice 
% and choose opposite value for ,
% trained classifier output predict completely flipped class label.
% While we don't always know whether $\pi_+ \geq 1/2$ in practice,
% we can classify unseen data into two classes.
% Even in such cases, we can classify unseen data into two classes, 
% but we cannot know the mapping between true positive and negative class, and our predicted positive and negative class.
% That will satisfy our demand for most cases.
\section{Theoretical Analysis}
\label{sec:analysis}
In this section, we analyze the generalization bound for our algorithms.
As the first step, we show estimation error bounds for SU, DU, and SD classification via Rademacher complexity.
\begin{comment}
For simplicity, we consider following three risks for SDU classification separately, where one of $\{ \gamma_1,\gamma_2,\gamma_3 \}$ in Eq.\,\eqref{eq:risk_sdu} is set to zero.
\begin{align}
\label{eq:risk_SDSU}
    & R_{\rm SDSU}^{\gamma}(f) \defeq (1-\gamma) R_{\rm SD}(f) + \gamma R_{\rm SU}(f), \\
\label{eq:risk_SDDU}
    & R_{\rm SDDU}^{\gamma}(f) \defeq (1-\gamma) R_{\rm SD}(f) + \gamma R_{\rm DU}(f), \\
\label{eq:risk_SUDU}
    & R_{\rm SUDU}^{\gamma}(f) \defeq (1-\gamma) R_{\rm SU}(f) + \gamma R_{\rm DU}(f).
\end{align}
We call the training with these risks SDSU classification, SDDU classification, and SUDU classification.
From the analysis of estimation error bounds for SU, DU and SD classification,
we clarify which algorithm among three SDU classification is the most promising.
\end{comment}
With those bounds, we compare each algorithm and then we clarify which algorithm 
among the three SDU classification approaches is the most promising.
Finally, we show the estimation error bound for SDU classification.

\subsection{Estimation Error Bounds for SU, DU, and SD}
\label{sec:bounds_sudusd}
First, we investigate estimation error bounds for SU, DU and SD classification.
Let $\mathcal{F} \subset \mathbb{R}^\mathcal{X}$ be a function class of the specified model.
\begin{definition}[Rademacher Complexity]
Let $n$ be a positive integer,
$Z_1, \dots, Z_n$ be i.i.d. random variables drawn from a probability distribution with density $\mu$,
$\mathcal{H}=\{ h : \mathcal{Z} \rightarrow \mathbb{R} \}$be a class of measurable functions,
and $\mathbf{\sigma} = (\sigma_1, \dots, \sigma_n)$ be Rademacher variables,
i.e., random variables taking $+1$ and $-1$ with even probabilities.
Then the (expected) Rademacher complexity of $\mathcal{H}$ is defined as
\begin{equation}
    \mathfrak{R}(\mathcal{H}; n , \mu) 
    \defeq \E_{Z_1, \dots, Z_n \sim \mu} \E_{\mathbf{\sigma}} 
    \left[ \sup_{h \sim \mathcal{H}} \frac{1}{n} \sum_{i=1}^{n} \sigma_i h(Z_i) \right].
\end{equation}
\end{definition}
For the function class $\mathcal{F}$ and any probability density $\mu$, we assume
\begin{equation}
\label{eq:assumption}
\mathfrak{R}(\mathcal{F}; n , \mu) \leq \frac{C_{\mathcal{F}}}{ \sqrt{n}}.
\end{equation}
This assumption holds for many models such as linear-in-parameter model class $\mathcal{F}=\left\{ f(\bx) = \w^\top \boldsymbol{\phi}(\bx) \right\}$.

Partially based on \cite{bao2018classification}, we have estimation error bounds for SU, DU and SD classification as follows.
\begin{theorem}
\label{theo:est}
Let $R(f)=\E[\ell(f(\bx), y)]$ be a classification risk for function $f$, $f^*$ be its minimizer 
and $\hat{f}_{\rm SU}, \hat{f}_{\rm DU}, \hat{f}_{\rm SD}$ be minimizers of empirical SU, DU, SD risk, respectively. 
Assume the the loss function $\ell$ is $\rho$-$Lipschitz$ function with respect to the first argument ($0 < \rho < \infty$), 
and all functions in the model class $\mathcal{F}$ are bounded, 
i.e., there exists a constant $C_b$ such that $\| f \| \leq C_b$ for any $f \in \mathcal{F}$.
Let $C_\ell \defeq \sup_{t \in \{\pm1 \}} \ell(C_b, t)$. 
For any $\delta > 0$, all of the following inequalities hold with probability at least $1-\delta$:
\begin{align}
    & R(\hat{f}_\SU) - R(f^*) \leq C_{\mathcal{F, \ell, \delta}}\left(\frac{2 \pi_\S}{\sqrt{2 n_\S}} + \frac{1}{\sqrt{n_\U}}\right), \\
    & R(\hat{f}_\DU) - R(f^*) \leq C_{\mathcal{F, \ell, \delta}}\left(\frac{2 \pi_\D}{\sqrt{2 n_\D}} + \frac{1}{\sqrt{n_\U}}\right), \\
    & R(\hat{f}_\SD) - R(f^*) \leq C_{\mathcal{F, \ell, \delta}}\left(\frac{\pi_\S}{\sqrt{2 n_\S}} + \frac{\pi_\D}{\sqrt{2 n_\D}}\right),
\end{align}
where
\begin{align}
 & C_{\mathcal{F, \ell, \delta}} = \frac{1}{|\pi_+ - \pi_-|} \left (4\rho C_{\mathcal{F}} + \sqrt{2 C_{\ell}^2 \log \frac{8}{\delta}} \right).
\end{align}
\end{theorem}

\subsection{Comparison of SD, SU, and DU Bounds}
\label{sec:theo_sudusd}
Here, we compare SU, DU, and SD algorithms from the point of view of estimation error bounds.
Under the generation process of similar and dissimilar pairs in Eq.\,\eqref{eq:gen_process_sd},
we have
% Since prior $\pi_\S$ is always larger than $\pi_\D$, we have following corollary.
\begin{corollary}
Suppose similar and dissimilar pairs follow the generation process in Eq.\,\eqref{eq:gen_process_sd}.
% Suppose the number of similar and dissimilar are determined by binomial distribution, i.e., $n_\S \sim {\rm Binomial}(n_\SD, \pi_\S)$.
We denote estimation error bounds for SD, SU, and DU in Theorem~\ref{theo:est} by $V_{\rm SD}$, $V_{\rm SU}$, and $V_{\rm DU}$, respectively.
Then, $V_\DU \leq V_\SU$ and $V_\SD \leq V_\SU$ hold 
with the probability at least $1-\exp(-c n_\SD)$ for some constant $c>0$.
%\begin{equation}
%    1 - \exp \left(- \frac{n_\SD \pi_\D}{2} \left( 1 - \frac{\pi_\D}{\pi_\S^2 + \pi_\D^2} \right)^2 \right).
%\end{equation}
\end{corollary}

\begin{proof}
If ${\pi_\S} / \sqrt{2 n_\S} > {\pi_\D} / \sqrt{2 n_\D}$ holds, then
\begin{equation*}
    \frac{V_{\rm SU} - C_{\mathcal{F, \ell, \delta}}/{\sqrt{n_\U}}}{V_{\rm DU} - {C_{\mathcal{F, \ell, \delta}}} / {\sqrt{n_\U}}}
    = \frac{\pi_\S / \sqrt{2 n_\S}}{\pi_\D / \sqrt{2 n_\D}}
    > 1,
\end{equation*}
and
\begin{equation*}
V_\SU - V_\SD
= C_{\mathcal{F, \ell, \delta}}\left( \frac{\pi_\S}{\sqrt{2 n_\S}} - \frac{\pi_\D}{\sqrt{2 n_\D}} + \frac{1}{\sqrt{n_\U}}\right)
> 0.
\end{equation*}
These two inequalities indicate $V_\DU \leq V_\SU$ and $V_\SD \leq V_\SU$, respectively. \\
Since we assume the generation process in Eq.\eqref{eq:gen_process_sd},
the class of each pair (i.e., similar or dissimilar) follows a Bernoulli distribution. 
Therefore, the number of pairs in each class follows a binomial distribution, namely,
$n_\D \sim Binomial(n_\SD, \pi_\D)$.
By using Chernoff's inequality in \cite{okamoto1959some}, we have
\begin{equation*}
\begin{split}
& p\left(\frac{\pi_\S}{\sqrt{2 n_\S}} \leq \frac{\pi_\D}{\sqrt{2 n_\D}}\right)
 = p\left(n_\D \leq \frac{ n_\SD \pi_\D^2}{\pi_\S^2 + \pi_\D^2}\right) 
  \leq \exp \left(- \frac{n_\SD \pi_\D}{2(1-\pi_\D)} \left( 1 - \frac{\pi_\D}{\pi_\S^2 + \pi_\D^2} \right)^2 \right).
\end{split}
\end{equation*}
Therefore, 
\begin{equation*}
\begin{split}
    & p(V_\DU \leq V_\SU \land V_\SD \leq V_\SU)
     \geq 1 - \exp \left(- \frac{n_\SD \pi_\D}{2(1-\pi_\D)} \left( 1 - \frac{\pi_\D}{\pi_\S^2 + \pi_\D^2} \right)^2 \right).
\end{split}
\end{equation*}
\end{proof}

From above discussion, 
when $n_\SD$ is sufficiently large, 
we have $V_\SD \leq V_\DU \leq V_\SU$ or $V_\DU \leq V_\SD \leq V_\SU$ with high probability.
Thus, we claim that the best pairwise combination for SDU algorithm is SDDU.

\subsection{Estimation Error Bounds for SDU}
\label{sec:bounds_sdu}
Now we consider the estimation error bound for SDU classification.
With the same technique in Theorem~\ref{theo:est}, we have the following theorem.
\begin{theorem}
\label{theo:est_sdu}
Let $R(f)=\E[\ell(f(\bx), y)]$ be a classification risk for function $f$, $f^*$ be its minimizer 
and $\hat{f}_{\rm SDU}$ be a minimizer of 
the empirical risk $\widehat{R}_{\rm SDU}^{\boldsymbol{\gamma}}$ in Eq.\,\eqref{eq:emp_risk_sdu}.
Assume the the loss function $\ell$ is $\rho$-$Lipschitz$ function with respect to the first argument ($0 < \rho < \infty$), 
and all functions in the model class $\mathcal{F}$ are bounded, 
i.e., there exists an constant $C_b$ such that $\| f \| \leq C_b$ for any $f \in \mathcal{F}$.
Let $C_\ell \defeq \sup_{t \in \{\pm1 \}} \ell(C_b, t)$. 
For any $\delta > 0$, with probability at least $1-\delta$,
\begin{equation}
\begin{split}
     R(\hat{f}_{\rm SDU}) - R(f^*) 
    & \leq C'_{\mathcal{F, \ell, \delta}}\left( 
    (2 \gamma_1 + \gamma_3) \frac{\pi_\S}{\sqrt{2 n_\S}} + (2 \gamma_2 + \gamma_3) \frac{\pi_\D}{\sqrt{2 n_\D}} \right. \\
    & \qquad\left. + (|\gamma_1 \pi_- - \gamma_2 \pi_+| + |\gamma_1 \pi_+ - \gamma_2 \pi_-|) \frac{1}{\sqrt{n_\U}}\right),
\end{split}
\end{equation}
where
\begin{equation}
  C'_{\mathcal{F, \ell, \delta}} = \frac{1}{|\pi_+ - \pi_-|} \left (4\rho C_{\mathcal{F}} + \sqrt{2 C_{\ell}^2 \log \frac{12}{\delta}} \right).
\end{equation}
\end{theorem}

\section{Experiments}
\label{sec:experiments}

In this section, we experimentally evaluate the performance of our SDU algorithm
and investigate the behaviors of SU, DU and SD classification.

\subsection{Datasets}
We conducted experiments on ten benchmark datasets obtained from UCI Machine Learning Repository~\cite{lichman2013uci} and LIBSVM~\cite{chang2011libsvm}.
To convert labeled data into similar and dissimilar pairs, 
we first determined the positive prior $\pi_+$.
Then we randomly subsampled pairwise similar and dissimilar data 
following the ratio of $\pi_\S$ and $\pi_D$.
To obtain unlabeled data, we randomly picked data following the ratio of $\pi_+$ and $\pi_-$.
For all experiments, $\pi_+$ was set to 0.7.

\subsection{Common Setup}
As a classifier, we used the linear-in-input model 
$f(\bx) = \w^\top \bx + b$.
The weight of L2 regularization was chosen from $\{10^{-1}, 10^{-4}, 10^{-7}\}$.
For SDU algorithms, the combination parameter $\gamma$ was chosen from $\{0.0, 0.2, ..., 0.8, 1.0 \}$.
For hyper-parameter tuning, we used 5-fold cross-validation.
To estimate validation error, the empirical risk on SD data $\widehat{R}_\SD$ equipped with the zero-one loss $\ell(\cdot) = (\frac{1}{2}(1- {\rm sign}(\cdot)))$ was used.
In each trial, the parameters with the minimum validation error were chosen.

We used squared loss for experiments in Sec.\,\ref{sec:sudusd} and Sec.\,\ref{sec:sdsdu} and both squared and double-hinge loss for experiments in Sec.\,\ref{sec:sdu_comp}.
For experiments in Sec.\,\ref{sec:sdsdu} and Sec.\,\ref{sec:sdu_comp}, 
class prior $\pi_+$ was estimated from the number of similar and dissimilar pairs by means of Eq.\,\eqref{eq:cpe}.

\begin{figure}[bthp]
 \centering
 \subfloat[adult]{\includegraphics[scale=.45]{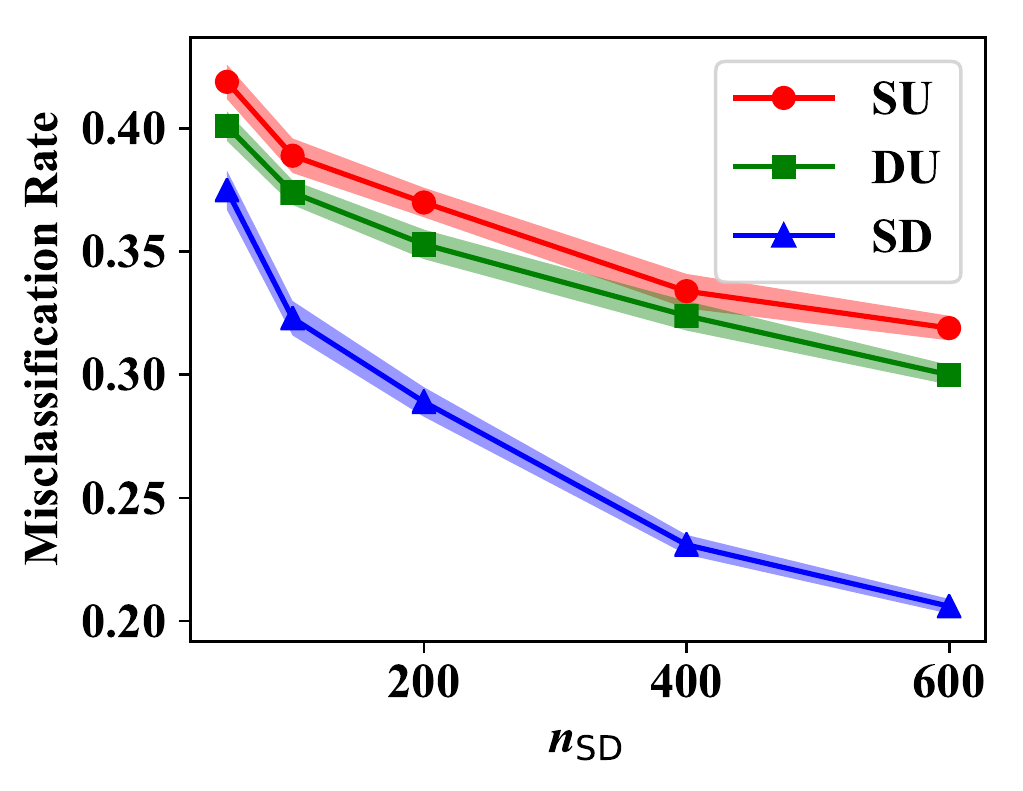}}
 \subfloat[phishing]{\includegraphics[scale=.45]{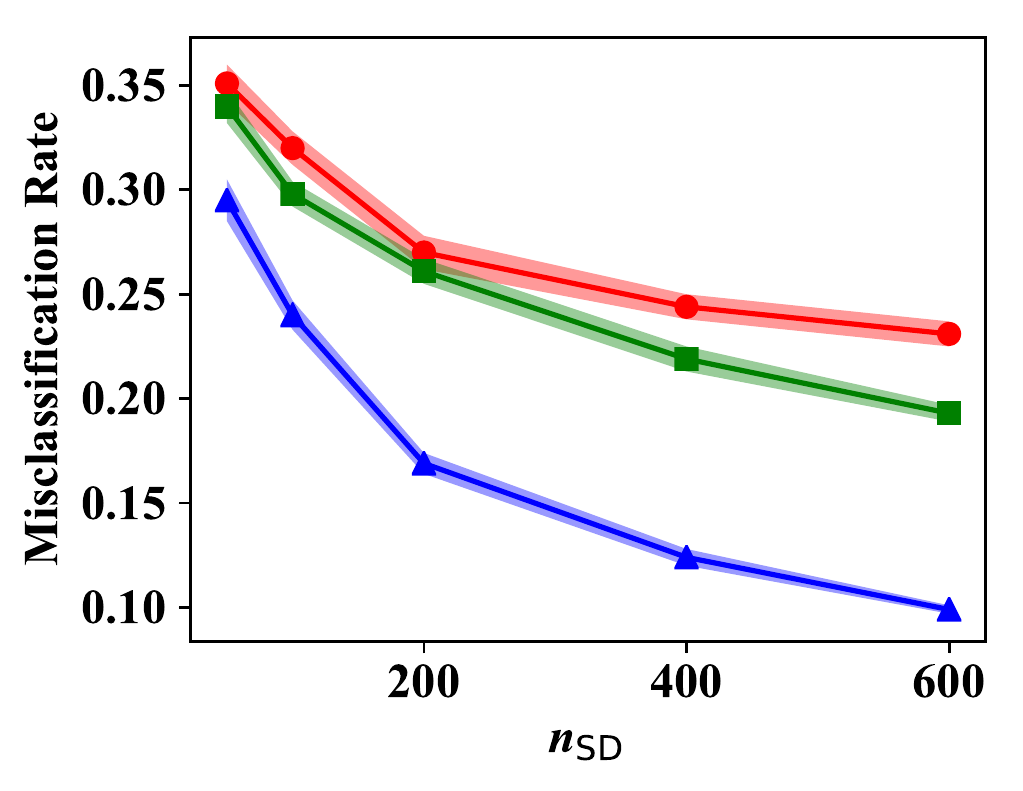}}
 \caption{
 Average misclassification rate and standard error as a function of the number of similar and dissimilar pairs over 50 trials.
 For all experiments, class prior $\pi_+$ is set to 0.7 and $n_\U$ is set to 500.
 }
 \label{fig:sudusd}
\end{figure}

\subsection{Comparison of SU, DU, and SD Performances}
\label{sec:sudusd}
We compared the performances of SU, DU and SD classification.
We set the number of unlabeled samples to 500 and the number of pairwise data to $\{50, 100, 200, 400, 600\}$.
In these experiments,
we assumed true class prior $\pi_+$ is known.
As we show the results in Fig.\,\ref{fig:sudusd}, DU and SD classification consistently outperform SU classification.
The results are consistent with our analysis in Sec.\,\ref{sec:theo_sudusd}
that SD and DU classification are likely to outperform SU classification.\footnote{Due to limited space, the magnified versions of experimental results are shown in Appendix \ref{sec_a:exp}.}

\subsection{Improvement by Unlabeled Data}
\label{sec:sdsdu}
We investigated the effect of unlabeled data on classification performance.
The number of pairwise data was set to 50.
We compared the performance of three SDU classification methods and SD classification.
As the results are shown in Fig.\,\ref{fig:sdsdu},
when the number of unlabeled data is sufficiently large, 
SDDU classification outperforms SD classification.
Furthermore, we demonstrate that SDDU classification constantly performs the best among all SDU classification.

\begin{figure}[bthp]
\centering
\subfloat[adult]{\includegraphics[scale=.45]{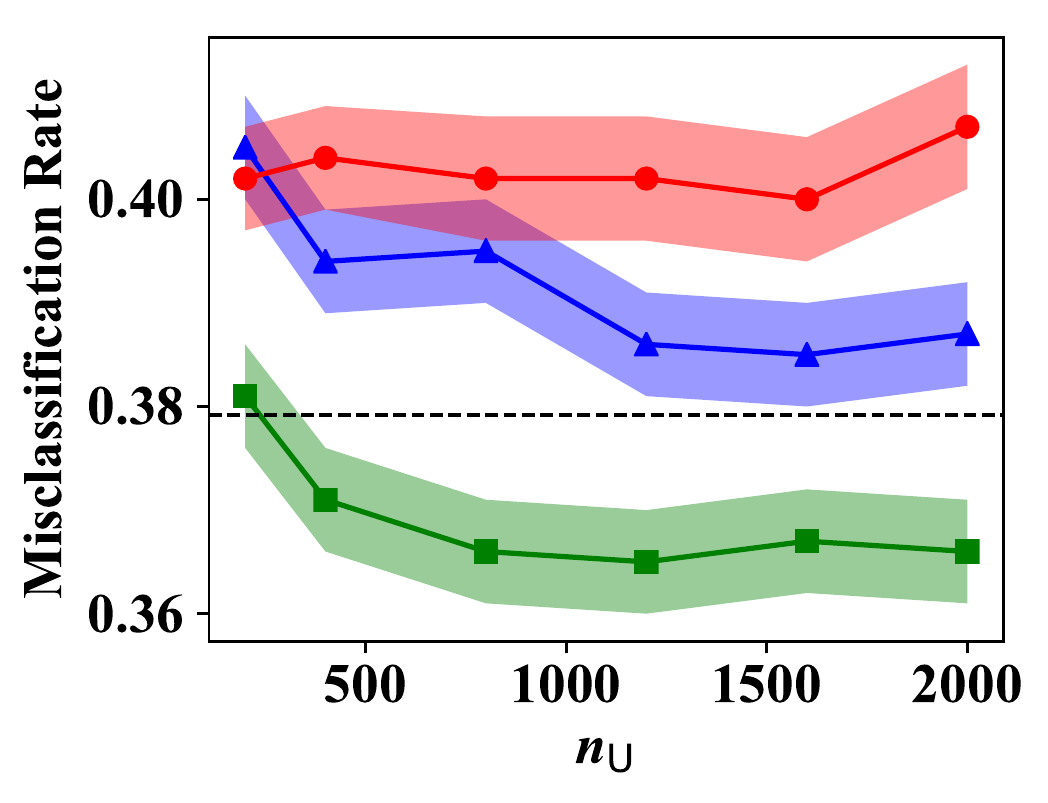}}
\subfloat[phishing]{\includegraphics[scale=.45]{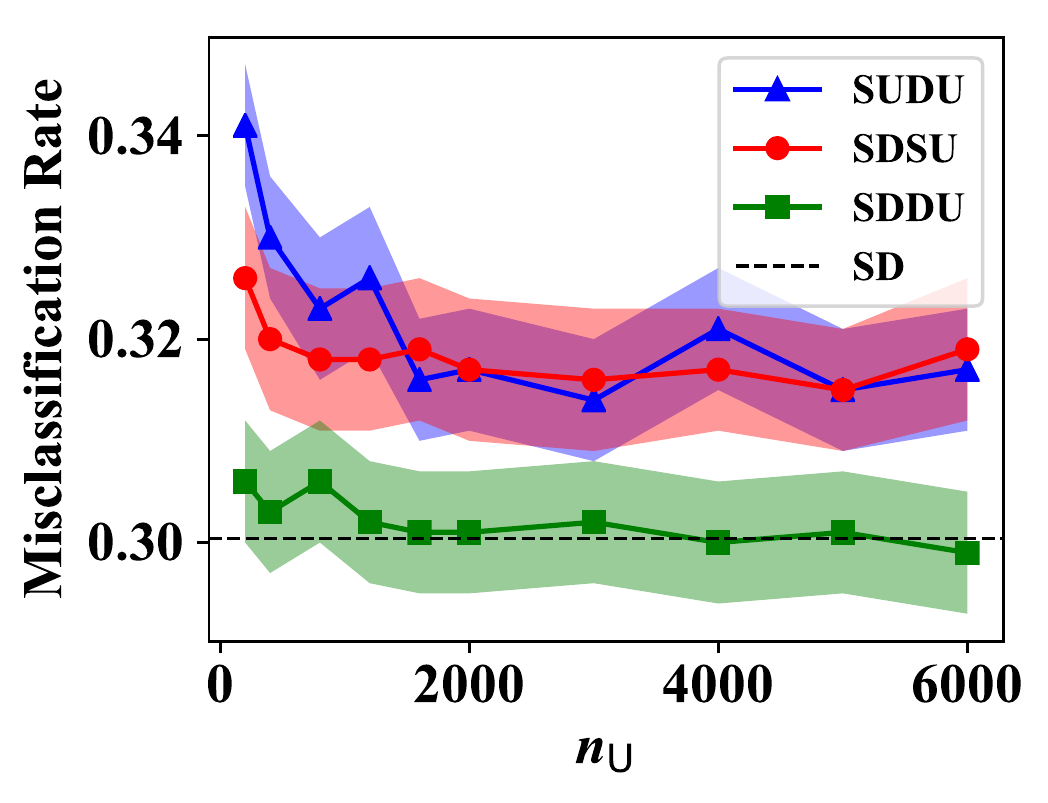}} \\
\vspace{-2mm}
\subfloat[spambase]{\includegraphics[scale=.45]{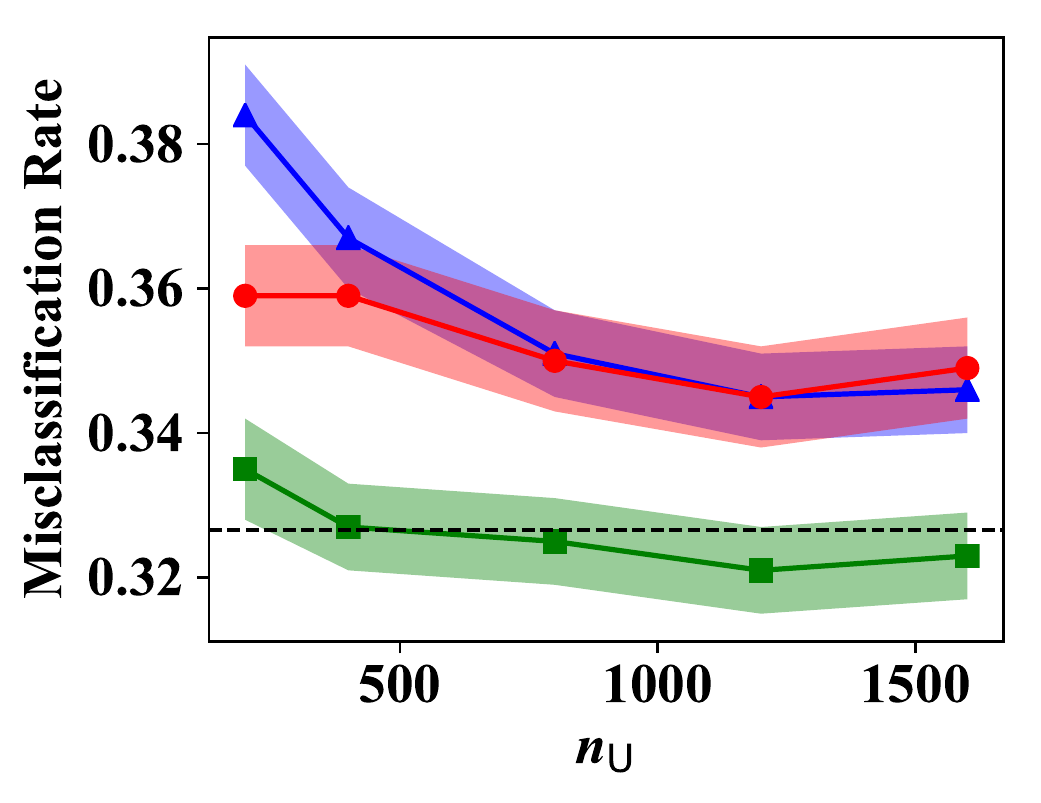}}
\subfloat[w8a]{\includegraphics[scale=.45]{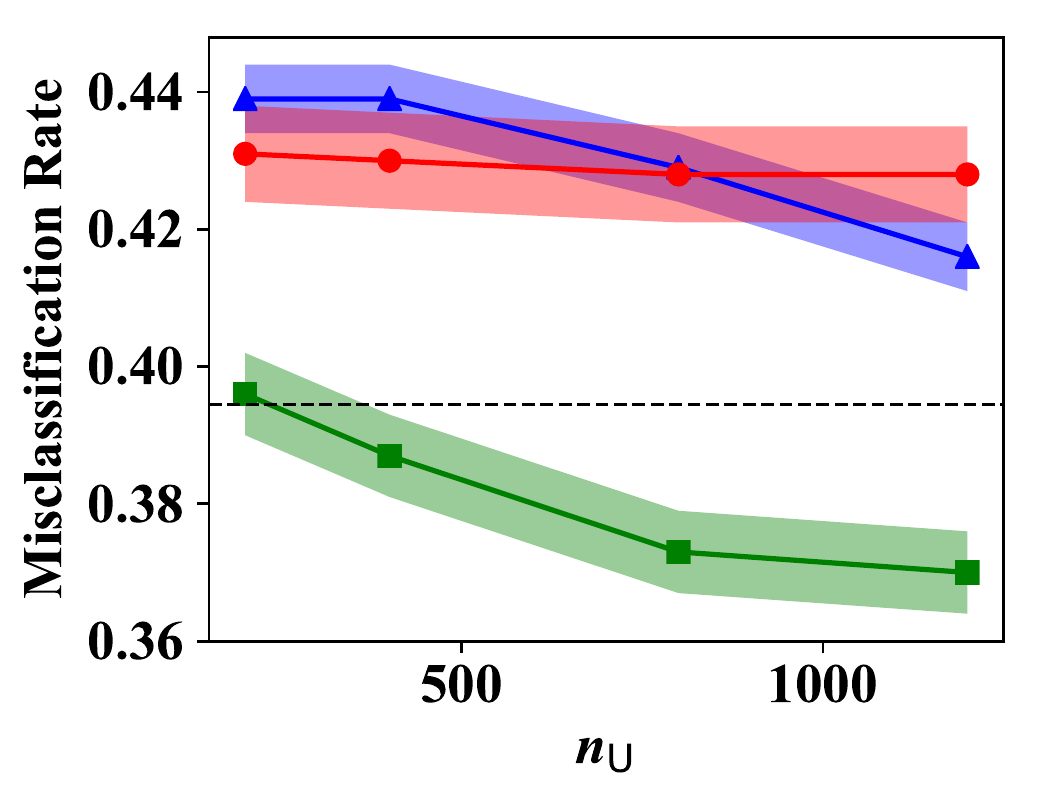}}
 \caption{
 Average misclassification rate and standard error as a function of the number of unlabeled samples over 100 trials.
 For all experiments, the class prior $\pi_+$ is set to 0.7 and $n_\SD$ is set to 50.
 The mean error rate of SD classification is drawn with a black dashed line.
 }
 \label{fig:sdsdu}
\end{figure}

\begin{table*}[bth]
\centering
\caption{
Mean accuracy and standard error on different benchmark datasets over 50 trials.
For all experiments, class prior $\pi_+$ is set to 0.7 and $n_\U$ is set to 500.
In SDU classification, we estimate class prior from $n_\S$ and $n_\D$.
For clustering algorithms, the performance is evaluated by clustering accuracy $1-\min(r, 1-r)$, where $r$ is error rate.
Bold numbers indicate outperforming methods, chosen by one-sided t-test with the significance level 5\%.
}
\scalebox{0.82}{
\label{table:SDU}
\begin{tabular}{@{\extracolsep{2pt}} cccccccc}
\hline
 &   &  \multicolumn{2}{c}{SDDU (proposed)}   &           \multicolumn{4}{c}{Baselines}    \\
  \cline{3-4} \cline{5-8}
  Dataset &  $n_\SD$ &        Squared &     Double-Hinge &          KM &         CKM &  SSP &  ITML \\
\hline
    adult &   50 &  61.9 (0.9) & \bf 77.7 (0.6) &  65.0 (0.8) &  66.6 (1.1) &  69.5 (0.3) &  62.4 (0.7) \\
    $d=123$ &  200 &  71.4 (0.7) & \bf 82.5 (0.3) &  63.3 (0.8) &  71.9 (0.9) &  69.3 (0.3) &  60.8 (0.7) \\\hline
   banana &   50 & \bf 63.9 (1.2) & \bf 63.5 (1.1) &  52.9 (0.4) &  52.7 (0.4) &  58.7 (0.7) &  53.0 (0.4) \\
   $d=2$ &  200 & \bf 66.5 (0.8) & \bf 66.9 (0.7) &  52.5 (0.2) &  52.5 (0.2) & \bf 66.5 (1.3) &  52.5 (0.2) \\\hline
   codrna &   50 & \bf 78.1 (1.1) &  68.5 (0.8) &  62.6 (0.5) &  61.5 (0.4) &  54.6 (1.0) &  62.7 (0.5) \\
   $d=8$ &  200 & \bf 87.7 (0.6) &  72.7 (0.7) &  62.8 (0.5) &  59.5 (0.5) &  53.2 (0.7) &  62.5 (0.5) \\\hline
   ijcnn1 &   50 &  64.7 (0.8) & \bf 68.8 (0.9) &  55.5 (0.6) &  54.7 (0.5) &  60.9 (0.8) &  55.8 (0.6) \\
   $d=22$ &  200 & \bf 75.1 (0.7) & \bf 76.3 (0.5) &  54.2 (0.3) &  53.1 (0.3) &  59.5 (0.8) &  54.3 (0.4) \\\hline
    magic &   50 & \bf 65.5 (0.9) & \bf 65.1 (1.0) &  52.4 (0.2) &  51.8 (0.2) &  52.7 (0.3) &  52.5 (0.2) \\
    $d=10$ &  200 & \bf 73.0 (0.6) &  71.4 (0.7) &  52.0 (0.2) &  51.7 (0.2) &  52.6 (0.3) &  52.0 (0.2) \\\hline
 phishing &   50 &  69.4 (0.8) & \bf 80.5 (0.9) &  62.6 (0.3) &  62.6 (0.3) &  68.1 (0.3) &  62.5 (0.3) \\
 $d=68$ &  200 &  81.7 (0.7) & \bf 87.0 (0.4) &  62.6 (0.3) &  62.8 (0.3) &  68.4 (0.3) &  62.6 (0.3) \\\hline
  phoneme &   50 & \bf 67.9 (0.9) & \bf 69.2 (0.9) & \bf 67.8 (0.3) & \bf 68.9 (0.5) & 66.5 (1.0) & \bf 67.8 (0.3) \\
  $d=5$ &  200 & \bf 73.5 (0.5) & \bf 74.4 (0.4) &  67.8 (0.3) &  71.0 (0.6) &  72.0 (0.7) &  67.9 (0.3) \\\hline
 spambase &   50 &  66.7 (0.8) & \bf 82.9 (0.6) &  63.7 (1.1) &  64.2 (1.1) &  70.4 (0.3) &  61.4 (1.1) \\
 $d=57$ &  200 &  77.9 (0.7) & \bf 87.5 (0.3) &  61.8 (1.2) &  70.4 (0.6) &  70.7 (0.3) &  60.6 (1.2) \\\hline
      w8a &   50 &  60.8 (0.9) & \bf 73.2 (0.9) &  69.3 (0.3) &  66.0 (0.7) &  64.2 (0.6) &  69.0 (0.3) \\
      $d=300$ &  200 &  64.1 (0.7) & \bf 80.2 (0.7) &  69.3 (0.3) &  56.3 (0.6) &  67.2 (0.7) &  68.6 (0.3) \\\hline
 waveform &   50 &  72.9 (1.0) & \bf 83.2 (1.0) &  51.5 (0.2) &  51.6 (0.2) &  53.3 (0.3) &  51.5 (0.2) \\
 $d=21$ &  200 &  83.2 (0.6) & \bf 86.7 (0.6) &  51.5 (0.2) &  51.5 (0.1) &  53.1 (0.3) &  51.5 (0.2) \\
\hline
\end{tabular}}
\end{table*}

\subsection{Comparison of SDU and Existing Methods}

\label{sec:sdu_comp}
We evaluated the performances of the proposed SDU classification with four baseline methods.
We conducted experiments on each benchmark dataset with 500 unlabeled data and 
$\{50, 200\}$ similar or dissimilar pairs in total.
Accuracy was measured in each trial with 500 test samples.
%Since we need class prior $\pi_+$ for our SDU classification,
%we estimate it from the number of similar and dissimilar data with the way in Eq.\,\eqref{eq:cpe}.
Due to limited space, 
we show only the results of SDDU classification as a representative of the proposed method.
As we can see in Table~\ref{table:SDU}, SDDU classification performs the best on many datasets.
The details of the baseline methods are described below.

\vspace{2mm}
{\bf KMeans Clustering (KM): }
Ignoring all pairwise information, 
K-means clustering algorithm~\cite{macqueen1967some} is applied to only training data.
We predicted labels of test data with learned clusters.

{\bf Constrained KMeans Clustering (CKM): }
Constrained K-means clustering~\cite{wagstaff2001constrained} is a clustering method using pairwise similar / dissimilar information as must-link / cannot-link constraints.

{\bf Semi-supervised Spectral Clustering (SSP): }
Semi-supervised spectral clustering~\cite{chen2012spectral} is a spectral clustering based method, 
where similar and dissimilar pairs are used for affinity propagation.
We set $k=5$, which is used for k-nearest-neighbors graph construction, and $\sigma^2=1$,
which is a precision parameter for similarity measurement.

{\bf Information Theoretical Metric Learning (ITML): }
Information Theoretical Metric Learning~\cite{davis2007information} is a metric learning based algorithm, where similar and dissimilar pairs used for regularizing the covariance matrix.
We used the identity matrix as prior information and a slack parameter $\gamma$ was set to 1.
For test samples prediction, 
k-means clustering was applied with the obtained metric.

For clustering algorithms,
the number of clusters $K$ was set to 2.
To evaluate the performances of k-means based clustering methods (i.e., KM, CKM, and ITML),
test samples were completely separated from training samples.
The labels of test samples are predicted based on the clusters obtained only from training samples.
For semi-supervised spectral clustering, we applied the algorithm on both train and test samples
so that we could predict for test samples.

\section{Conclusion}
In this paper, we proposed a novel weakly supervised classification algorithm, 
which is the empirical risk minimization from pairwise similar, dissimilar, and unlabeled data.
We formulated the optimization problem for SDU classification and provided practical solutions with squared and double-hinge loss.
From estimation error bound analysis, we show that the SDDU combination is the most promising for SDU classification.
Through experiments on benchmark dataset,
we confirmed that our SDU classification outperforms baseline methods.
\section*{Acknowledgments}
HB was supported by JST ACT-I Grant Number JPMJPR18UI.
IS was supported by JST CREST Grant Number JPMJCR17A1, Japan.
MS was supported by JST CREST Grant Number JPMJCR1403.
\clearpage
\bibliographystyle{abbrvnat}
\bibliography{ref.bib}
\clearpage
\appendix
\section{Proofs of Theorems}
\label{sec_a:proofs}
In this section, we give complete proofs in Secs. \ref{sec:propose} and \ref{sec:analysis}.

\subsection{Preliminaries}
\label{sec_a:pre}
For convenience,
we introduce pointwise densities $\widetilde{p}_\S(\bx)$ and $\widetilde{p}_\D(\bx)$ for similar and dissimilar data.
By marginalizing pairwise densities $p_\S(\bx, \bx')$ and $p_\D(\bx, \bx')$ by $\bx'$,
we have 
\begin{align}
    & \widetilde{p}_\S(\bx) \defeq \int p_\S(\bx, \bx') d\bx' = \frac{\pi_+^2}{\pi_+^2 + \pi_-^2} p_+(\bx) + \frac{\pi_-^2}{\pi_+^2 + \pi_-^2} p_-(\bx), \\
    & \widetilde{p}_\D(\bx) \defeq \int p_\D(\bx, \bx') d\bx' = \frac{1}{2} p_+(\bx) + \frac{1}{2} p_-(\bx).
\end{align}
Let $\widetilde{\mathfrak{D}}_\S$ be a set of pointwise samples in $\mathfrak{D}_\S$ and $\widetilde{\mathfrak{D}}_\D$ be a set of pointwise samples in $\mathfrak{D}_\D$ as well.
\begin{align}
    & \widetilde{\mathfrak{D}}_\S \defeq \{\widetilde{\bx}_{\S,i} \}_{i=1}^{2 n_\S}
    = \bigcup \{\bx_\S, \bx'_\S
    \mid (\bx_\S, \bx'_\S) \in \mathfrak{D}_\S\}, \\
    & \widetilde{\mathfrak{D}}_\D \defeq \{\widetilde{\bx}_{\D,i} \}_{i=1}^{2 n_\D}
    = \bigcup \{\bx_\D, \bx'_\D
    \mid (\bx_\D, \bx'_\D) \in \mathfrak{D}_\D\}.
\end{align}
Then, we can consider the generation process of pointwise similar/dissimilar data as
\begin{align*}
& \widetilde{\mathfrak{D}}_\S \sim \widetilde{p}_\S(\bx), \\
& \widetilde{\mathfrak{D}}_\D \sim \widetilde{p}_\D(\bx).
\end{align*}
We use above notations for the proofs of Theorems~\ref{theo:conv},~\ref{theo:est}, and \ref{theo:est_sdu}.

\subsection{Proof of Theorem~\ref{theo:dusd_risk}}
We start from an unbiased risk estimator from SU data in Proposition~\ref{theo:su}.
The classification risk is equivalently represented as:
\begin{equation}
\label{eq_a:su_risk}
    R_{\SU}(f) = \pi_\S \E_{(X,X') \sim p_\S(\boldsymbol{x}, \boldsymbol{x}')}
    \left[
            \frac{\widetilde{\mathcal{L}}(f(X)) + \widetilde{\mathcal{L}}(f(X'))}{2}
    \right] + \E_{X \sim p_\U(\boldsymbol{x})}
    \left[
            \mathcal{L}(f(X), -1)
    \right],
\end{equation}
where
\begin{align*}
    & \mathcal{L}(z, t) \defeq \frac{\pi_{+}}{\pi_{+} - \pi_{-}} \ell(z, t) - \frac{\pi_{-}}{\pi_{+} - \pi_{-}} \ell(z, -t), \\
    & \widetilde{\mathcal{L}} (z) \defeq \ell(z,+1) - \ell(z,-1).
\end{align*}
Since the pairwise density can be decomposed into density of similar pairs and dissimilar pairs, namely,
$p(\bx, \bx') = \pi_\S p_\S(\bx, \bx') + \pi_\D p_\D(\bx, \bx')$,
the expectation over $p(\bx, \bx')$ can be decomposed as follows.
\begin{equation}
\label{eq_a:usd}
    \E_{(X,X') \sim p(\boldsymbol{x}, \boldsymbol{x}')} [\cdot]
    = \pi_\S \E_{(X,X') \sim p_\S(\boldsymbol{x}, \boldsymbol{x}')}[\cdot]
    + \pi_\D \E_{(X,X') \sim p_\D(\boldsymbol{x}, \boldsymbol{x}')}[\cdot].
\end{equation}
In addition, the risk over $p_\U(\bx)$ can be equivalently represented as:
\begin{align}
\label{eq_a:urisk_decomp}
    & \E_{X \sim p_\U(\boldsymbol{x})}
    \left[
            \mathcal{L}(f(X), +1) 
    \right]
    = \E_{(X,X') \sim p(\boldsymbol{x}, \boldsymbol{x}')}
    \left[
            \frac{\mathcal{L}(f(X), +1) + \mathcal{L}(f(X'), +1)}{2} 
    \right], \\
        & \E_{X \sim p_\U(\boldsymbol{x})}
    \left[
            \mathcal{L}(f(X), -1) 
    \right]
    = \E_{(X,X') \sim p(\boldsymbol{x}, \boldsymbol{x}')}
    \left[
            \frac{\mathcal{L}(f(X), -1) + \mathcal{L}(f(X'), -1)}{2} 
    \right].
\end{align}
By applying Eq.\eqref{eq_a:usd},\eqref{eq_a:urisk_decomp},\,\eqref{eq_a:urisk_decomp} to \eqref{eq_a:su_risk},
we can derive the risk only from dissimilar and unlabeled distributions or similar and dissimilar distributions.
\begin{equation}
\begin{split}
    R_{\SU}(f) 
    & = \left( \E_{(X,X') \sim p(\boldsymbol{x}, \boldsymbol{x}')}
    \left[ \frac{\widetilde{\mathcal{L}}(f(X)) + \widetilde{\mathcal{L}}(f(X'))}{2}
    \right] \right. \\ 
    & \left. \qquad - \pi_\D \E_{(X,X') \sim p_\D(\boldsymbol{x}, \boldsymbol{x}')}
    \left[
            \frac{\widetilde{\mathcal{L}}(f(X)) + \widetilde{\mathcal{L}}(f(X'))}{2}
    \right] \right)\\
    &  \qquad + \E_{(X,X') \sim p(\boldsymbol{x}, \boldsymbol{x}')}
    \left[ \frac{\mathcal{L}(f(X), -1) + \mathcal{L}(f(X'), -1)}{2} 
    \right] \\
    & = \pi_\D \E_{(X,X') \sim p_\D(\boldsymbol{x}, \boldsymbol{x}')}
    \left[
            - \frac{\widetilde{\mathcal{L}}(f(X)) + \widetilde{\mathcal{L}}(f(X'))}{2}
    \right] \\ 
    & + \E_{(X,X') \sim p(\bx, \bx')}
    \left[
            \frac{\mathcal{L}(f(X), +1) + \mathcal{L}(f(X'), +1)}{2}
    \right] \\
    & = R_\DU(f),
\end{split}
\end{equation}
\begin{equation}
\begin{split}
    R_{\SU}(f) 
    & = \pi_\S \E_{(X,X') \sim p_\S(\boldsymbol{x}, \boldsymbol{x}')}
    \left[
            \frac{\widetilde{\mathcal{L}}(f(X)) + \widetilde{\mathcal{L}}(f(X'))}{2}
    \right] \\
    & \qquad + \left(\pi_\S \E_{(X,X') \sim p_\S(\boldsymbol{x}, \boldsymbol{x}')}
    \left[
            \frac{\mathcal{L}(f(X), -1) + \mathcal{L}(f(X'), -1)}{2} 
    \right] \right. \\
    & \qquad \qquad  \left. + \pi_\D \E_{(X,X') \sim p_\D(\boldsymbol{x}, \boldsymbol{x}')}
    \left[
            \frac{\mathcal{L}(f(X), -1) + \mathcal{L}(f(X'), -1)}{2} 
    \right]\right) \\
        & = \pi_\S \E_{(X,X') \sim p_\S(\boldsymbol{x}, \boldsymbol{x}')}
    \left[
            \frac{\mathcal{L}(f(X), +1) + \mathcal{L}(f(X'), +1)}{2} 
    \right] \\
    & \qquad + \pi_\D \E_{(X,X') \sim p_\D(\boldsymbol{x}, \boldsymbol{x}')}
    \left[
            \frac{\mathcal{L}(f(X), -1) + \mathcal{L}(f(X'), -1)}{2} 
    \right] \\
    & = R_\SD(f),
\end{split}
\end{equation}
where we use the equation $\mathcal{L}(z, +1) - \mathcal{L}(z, -1) = \widetilde{\mathcal{L}}(z)$.
Therefore, $R_\DU$ and $R_\SD$ are also unbiased estimators of the classification risk.
\qed

\subsection{Proof of Theorem~\ref{theo:conv}}
\begin{comment}
In \cite{bao2018classification}, they showed the Hessian of $\widehat{R}_\SU(\w)$ with respect to $\w$ is a
positive semidefinite matrix.
In the same way, we can show that the Hessian matrices of $\widehat{R}_\DU(\w)$ and $\widehat{R}_\SD(\w)$ are also positive semidefinite.
Thus, the Hessian of $\widehat{J}^{\boldsymbol{\gamma}}(\w)$ is positives semidefinite,
which indicates $\widehat{J}^{\boldsymbol{\gamma}}(\w)$ is convex with respect to $\w$.
\qed
\end{comment}
We prove this theorem based on the positive semidefiniteness of the Hessian matrix similarly to SU classification in \cite{bao2018classification}.
Since $\ell$ is a twice differentiable margin loss, there is a twice differentiable function $\psi:\R \rightarrow\R_+$
such that $\ell(z,t)=\psi(tz)$.
Here our objective function can be written as
\begin{equation}
\begin{split}
    \widehat{J}^{\boldsymbol{\gamma}}(\w) 
    & = \frac{\lambda}{2} \w^\top \w -\frac{\gamma_1 \pi_\S}{2 n_\S (\pi_+ - \pi_-)} \sum_{i=1}^{2 n_\S} \w^\top \boldsymbol{\phi}(\widetilde{\bx}_{\S,i})
    +\frac{\gamma_2 \pi_\D}{2 n_\D (\pi_+ - \pi_-)} \sum_{i=1}^{2 n_\D} \w^\top \boldsymbol{\phi}(\widetilde{\bx}_{\D,i}) \\
    & + \frac{\gamma_3 \pi_\S}{2 n_\S (\pi_+ - \pi_-)}  
    \sum_{i=1}^{2 n_\S} \left( \pi_+  \ell(\w^\top \boldsymbol{\phi}(\widetilde{\bx}_{\S,i}), + 1) - \pi_-  \ell(\w^\top \boldsymbol{\phi}(\widetilde{\bx}_{\S,i}), -1) \right) \\
    & - \frac{\gamma_3 \pi_\D}{2 n_\D (\pi_+ - \pi_-)}  
    \sum_{i=1}^{2 n_\D} \left( \pi_-  \ell(\w^\top \boldsymbol{\phi}(\widetilde{\bx}_{\D,i}), + 1) - \pi_+ \ell(\w^\top \boldsymbol{\phi}(\widetilde{\bx}_{\D,i}), -1) \right) \\
    & + \frac{1}{n_\U (\pi_+ - \pi_-)}  
    \sum_{i=1}^{n_\U} \left( (\gamma_2 \pi_+ - \gamma_1 \pi_-)  \ell(\w^\top \boldsymbol{\phi}(\bx_{\U,i}), + 1) + (\gamma_1 \pi_+ - \gamma_2 \pi_-)  \ell(\w^\top \boldsymbol{\phi}(\bx_{\U,i}), -1) \right).
\end{split}
\end{equation}
The second-order derivative of $\ell(z,t)$ with respect to $z$ can be computed as
\begin{equation}
    \frac{\partial^2 \ell(z,t)}{\partial z^2} 
    = \frac{\partial^2 \psi(tz)}{\partial z^2} 
    = t^2  \frac{\partial^2 \psi(\xi)}{\partial \xi^2} 
    = \frac{\partial^2 \psi(\xi)}{\partial \xi^2},
\end{equation}
where $\xi=tz$ is employed in the second equality and $t\in\{+1, -1\}$ is employed in the last equality.
Here, the Hessian of $\widehat{J}^{\boldsymbol{\gamma}}$ with respect to $\w$ is
\begin{equation}
\begin{split}
    \boldsymbol{H} \widehat{J}^{\boldsymbol{\gamma}}(\w)
    =  \lambda I 
    +   \frac{\partial^2 \psi(\xi)}{\partial \xi^2} & \left( \frac{\gamma_3}{2n_\S} \sum_{i=1}^{2 n_\S} \boldsymbol{\phi}(\widetilde{\bx}_{\S,i}) \boldsymbol{\phi}(\widetilde{\bx}_{\S,i})^\top
    + \frac{\gamma_3}{2n_\D} \sum_{i=1}^{2 n_\D}  \boldsymbol{\phi}(\widetilde{\bx}_{\D,i}) \boldsymbol{\phi}(\widetilde{\bx}_{\D,i})^\top \right. \\
    & \qquad \left. + \frac{\gamma_1 + \gamma_2}{n_\U} \sum_{i=1}^{n_\U} \boldsymbol{\phi}({\bx}_{\U,i}) \boldsymbol{\phi}({\bx}_{\U,i})^\top \right) \succeq 0,
\end{split}
\end{equation}
where $A \succeq 0$ means that a matrix $A$ is positive semidefinite.
Positive semidefiniteness of $\boldsymbol{H} \widehat{J}^{\boldsymbol{\gamma}}(\w)$ follows from 
$\frac{\partial^2 \psi(\xi)}{\partial \xi^2} \geq 0$ (\,$\because$~$\ell$ is convex) and 
$\boldsymbol{\phi}(\widetilde{\bx}) \boldsymbol{\phi}(\widetilde{\bx})^\top \succeq 0$.
Therefore, $\widehat{J}^{\boldsymbol{\gamma}}(\w)$ is convex with respect to $\w$.
\qed

\subsection{Proof of Theorem~\ref{theo:est}}
We apply the similar technique with SU classification to DU and SD classification.
From pointwise decomposition in Sec.\,\ref{sec_a:pre}, we have the following lemma.
\begin{lemma}
Given any function $f:\mathcal{X}\rightarrow$, we denote $R_{\rm \widetilde{S}U}$, $R_{\rm \widetilde{D}U}$, and $R_{\rm \widetilde{SD}}$ by
\begin{align}
   &  R_{\rm \widetilde{S}U} \defeq \pi_\S \E_{X \sim \widetilde{p}_\S(\boldsymbol{x})}
    \left[
            \widetilde{\mathcal{L}}(f(X))
    \right] + \E_{X \sim p_\U(\boldsymbol{x})}
    \left[
            \mathcal{L}(f(X), -1)
    \right], \\
&    R_{\rm \widetilde{D}U} \defeq \pi_\D \E_{X \sim \widetilde{p}_\D(\boldsymbol{x})}
    \left[
            - \widetilde{\mathcal{L}}(f(X))
    \right] + \E_{X \sim p_\U(\boldsymbol{x})}
    \left[
            \mathcal{L}(f(X), +1) 
    \right], \\
&   R_{\rm \widetilde{SD}}
    \defeq \pi_\S \E_{X \sim \widetilde{p}_\S(\boldsymbol{x})} \left[\mathcal{L}(f(X), +1) \right]
    + \pi_\D \E_{X \sim \widetilde{p}_\D(\boldsymbol{x})} \left[\mathcal{L}(f(X), -1) \right].
\end{align}
Then, $R_{\rm \widetilde{S}U}$, $R_{\rm \widetilde{D}U}$, and $R_{\rm \widetilde{SD}}$ are equivalent to $R_\SU$, $R_\DU$, and $R_\SD$, respectively.
\end{lemma}
Here, empirical versions of above risks are defined as:
\begin{align}
   &  \widehat{R}_{\rm \widetilde{S}U} 
   \defeq  \frac{\pi_\S}{2 n_\S} \sum_{i=1}^{2 n_\S} \widetilde{\mathcal{L}}(f(\widetilde{\bx}_{\S, i}))
 + \frac{1}{n_\U} \sum_{i=1}^{n_\U} \mathcal{L}(f\bx_{\U, i}), -1), \\
&    \widehat{R}_{\rm \widetilde{D}U} 
   \defeq  - \frac{\pi_\D}{2 n_\D} \sum_{i=1}^{2 n_\D} \widetilde{\mathcal{L}}(f(\widetilde{\bx}_{\D, i}))
 + \frac{1}{n_\U} \sum_{i=1}^{n_\U} \mathcal{L}(f(\bx_{\U, i}), +1), \\
&   \widehat{R}_{\rm \widetilde{SD}} 
   \defeq  \frac{\pi_\S}{2 n_\S} \sum_{i=1}^{2 n_\S} \mathcal{L}(f(\widetilde{\bx}_{\S, i}), +1)
   + \frac{\pi_\D}{2 n_\D} \sum_{i=1}^{2 n_\D} \mathcal{L}(f(\widetilde{\bx}_{\D, i}), -1).
\end{align}
Note that above empirical risks are equivalent to $\widehat{R}_\SU$, $\widehat{R}_\DU$, and $\widehat{R}_\SD$, respectively.

Now we show the uniform deviation bound, which is useful to derive estimation error bounds.
The proof can be found in the textbooks such as~\cite{mohri2018foundations}.
\begin{lemma}
\label{lemma_a:udb}
Let $Z$ be a random variable drawn from a probability distribution with density $\mu$,
$\mathcal{H}=\{h : \mathcal{Z}\rightarrow[0, M] \}(M>0)$ be a class of measurable functions, 
$\{z_i\}_{i=1}^n$ be i.i.d. samples drawn from the distribution with density $\mu$.
Then, for any $\delta > 0$, with the probability at least $1-\delta$,
\begin{equation}
    \sup_{h \in \mathcal{H}} \left| \E_{Z \sim \mu} [h(Z)] - \frac{1}{n}\sum_{i=1}^n h(z_i) \right|
    \leq 2 \mathfrak{R}(\mathcal{H}; \mu, n) + \sqrt{\frac{M^2 \log \frac{2}{\delta}}{2n}}.
\end{equation}
\end{lemma}
Here we can write the estimation error bound for SU classification by
\begin{equation}
\label{eq_a:udb_su}
    \begin{split}
        R(\widehat{f}_\SU) - R(f^*)
        & = R_{\SU}(\widehat{f}_\SU) - R_{\SU}(f^*) \\
        & \leq \left( R_{\SU}(\widehat{f}_\SU) - \widehat{R}_{\SU}(\widehat{f}_\SU) \right)
        + \left(\widehat{R}_{\SU}(f^*) - R_{\SU}(f^*)  \right) \\
        & \leq 2 \sup_{f \in \mathcal{F}} \left| R_{\SU}(f) - \widehat{R}_{\SU}(f) \right| \\
        & = 2 \sup_{f \in \mathcal{F}} \left| R_{\rm \widetilde{S}U}(f) - \widehat{R}_{\rm \widetilde{S}U}(f) \right| \\
        & = 2\pi_\S \sup_{f \in \mathcal{F}} 
        \left| \E_{X \sim \widetilde{p}_\S}\left[\widetilde{\mathcal{L}}(f(X))\right] - \frac{1}{2 n_\S} \sum_{i=1}^{2 n_\S}  \widetilde{\mathcal{L}}(f(\bx_{\S,i})) \right| \\
        & \qquad + 2 \sup_{f \in \mathcal{F}} \left| \E_{X \sim p_\U}\left[\mathcal{L}(f(X), -1)\right] - \frac{1}{n_\U} \sum_{i=1}^{n_\U} \mathcal{L}(f(\bx_{\U,i}), -1) \right|.
    \end{split}
\end{equation}
Similary, for DU and SD classification, we have
\begin{equation}
\label{eq_a:udb_du}
\begin{split}
         R(\widehat{f}_\DU) - R(f^*) 
    & \leq 2\pi_\D \sup_{f \in \mathcal{F}} 
        \left| \E_{X \sim \widetilde{p}_\D}\left[\widetilde{\mathcal{L}}(f(X))\right] - \frac{1}{2 n_\D} \sum_{i=1}^{2 n_\D}  \widetilde{\mathcal{L}}(f(\bx_{\D,i})) \right| \\
        & \qquad + 2 \sup_{f \in \mathcal{F}} \left| \E_{X \sim p_\U}\left[\mathcal{L}(f(X), +1)\right] - \frac{1}{n_\U} \sum_{i=1}^{n_\U} \mathcal{L}(f(\bx_{\U,i}), +1) \right|,
\end{split}
\end{equation}
\begin{equation}
\label{eq_a:udb_sd}
    \begin{split}
        R(\widehat{f}_\SD) - R(f^*) 
    & \leq 2\pi_\S \sup_{f \in \mathcal{F}} 
        \left| \E_{X \sim \widetilde{p}_\S}\left[\mathcal{L}(f(X), +1)\right] - \frac{1}{2 n_\S} \sum_{i=1}^{2 n_\S}  \mathcal{L}(f(\bx_{\S,i}), +1) \right| \\
        & \qquad + 2 \pi_\D \sup_{f \in \mathcal{F}} \left| \E_{X \sim \widetilde{p}_\D}\left[\mathcal{L}(f(X), -1)\right] - \frac{1}{2 n_\D} \sum_{i=1}^{2 n_\D} \mathcal{L}(f(\bx_{\D,i}), -1) \right|,
    \end{split}
\end{equation}

To derive bounds for each algorithm,
we derive the uniform deviation bound for $\widetilde{\mathcal{L}}(f(\cdot))$ and $\mathcal{L}(f(\cdot), \pm 1)$.
\begin{lemma}
\label{lm_a:udb_l}
Assume the the loss function $\ell$ is $\rho$-$Lipschitz$ function with respect to the first argument ($0 < \rho < \infty$), 
and all functions in the model class $\mathcal{F}$ are bounded, 
i.e., there exists a constant $C_b$ such that $\| f \| \leq C_b$ for any $f \in \mathcal{F}$.
Let $C_\ell \defeq \sup_{t \in \{\pm1 \}} \ell(C_b, t)$
and $\{\bx_i\}_{i=1}^n$ be i.i.d. samples drawn from a probability distribution with density $p$. 
For any $\delta > 0$, each of the following inequality holds with probability at least $1-\delta$.
\begin{align}
    & \sup_{f \in \mathcal{F}} \left| \E_{X \sim p}\left[\widetilde{\mathcal{L}}(f(X))\right] - \frac{1}{n} \sum_{i=1}^n  \widetilde{\mathcal{L}}(f(\bx_i)) \right|
    \leq \frac{4 \rho C_{\mathcal{F}} + \sqrt{2 C_\ell^2 \log \frac{4}{\delta}}}{|\pi_+ - \pi_-| \sqrt{n}}, \\
    & \sup_{f \in \mathcal{F}} \left| \E_{X \sim p}\left[\mathcal{L}(f(X), +1)\right] - \frac{1}{n} \sum_{i=1}^n \mathcal{L}(f(\bx_i), +1) \right|
    \leq \frac{2 \rho C_{\mathcal{F}} + \sqrt{\frac{1}{2} C_\ell^2 \log \frac{4}{\delta}}}{|\pi_+ - \pi_-| \sqrt{n}}, \\
    & \sup_{f \in \mathcal{F}} \left| \E_{X \sim p}\left[\mathcal{L}(f(X), -1)\right] - \frac{1}{n} \sum_{i=1}^n \mathcal{L}(f(\bx_i), -1) \right|
    \leq \frac{2 \rho C_{\mathcal{F}} + \sqrt{\frac{1}{2} C_\ell^2 \log \frac{4}{\delta}}}{|\pi_+ - \pi_-| \sqrt{n}}.
\end{align}
\end{lemma}
\begin{proof}
\begin{equation}
    \begin{split}
        & \sup_{f \in \mathcal{F}} \left| \E_{X \sim p}\left[\widetilde{\mathcal{L}}(f(X))\right] - \frac{1}{n} \sum_{i=1}^n  \widetilde{\mathcal{L}}(f(\bx_i)) \right| \\
        & \leq \frac{1}{|\pi_+ - \pi_-|} \underbrace{ \sup_{f \in \mathcal{F}} \left| \E_{X \sim p}\left[\ell(f(X),+1)\right] - \frac{1}{n} \sum_{i=1}^n  \ell(f(\bx_i),+1) \right|}_{\text{with the probability at least $1-\delta/2$}}  \\
        & \qquad + \frac{1}{|\pi_+ - \pi_-|} \underbrace{\sup_{f \in \mathcal{F}} \left| \E_{X \sim p}\left[\ell(f(X),-1)\right] - \frac{1}{n} \sum_{i=1}^n  \ell(f(\bx_i),-1) \right|}_{\text{with the probability at least $1-\delta/2$}}  \\
        & \leq \frac{1}{|\pi_+ - \pi_-|}\left\{ 4 \mathfrak{R}(\ell \circ \mathcal{F}; n, p) 
            + \sqrt{\frac{2 C_{\ell}^2 \log \frac{4}{\delta}}{n}} \right\},
    \end{split}
\end{equation}
where $\ell \circ \mathcal{F}$ in the last line means the class $\{\ell \circ f \mid f \in \mathcal{F} \}$.
The last inequality holds from Lemma~\ref{lemma_a:udb}, with the probability $1-\frac{\delta}{2}$ for each term.
By applying Taragrand's lemma,
\begin{equation}
\mathfrak{R}(\ell \circ \mathcal{F}; n, p) \leq \rho \mathfrak{R}(\mathcal{F}; n, p).
\end{equation}
With the assumption in Eq.\,\eqref{eq:assumption}, we obtain
\begin{equation}
\begin{split}
          \sup_{f \in \mathcal{F}} \left| \E_{X \sim p}\left[\widetilde{\mathcal{L}}(f(X))\right] - \frac{1}{n} \sum_{i=1}^n  \widetilde{\mathcal{L}}(f(\bx_i)) \right| 
         & \leq \frac{1}{|\pi_+ - \pi_-|}\left\{4\rho\frac{C_{\mathcal{F}}}{\sqrt{n}} + \sqrt{\frac{2 C_{\ell}^2 \log \frac{4}{\delta}}{n}} \right\}, \\
      & = \frac{4\rho C_{\mathcal{F}} + \sqrt{2 C_{\ell}^2 \log \frac{4}{\delta}}}{|\pi_+ - \pi_-|\sqrt{n}}.
\end{split}
\end{equation}
The bounds for $\mathcal{L}(f(\cdot), \pm 1)$ can be proven similarly to $\widetilde{\mathcal{L}}(f(\cdot))$.
\end{proof}

By combining Lemma~\ref{lm_a:udb_l} and Eqs.\,\eqref{eq_a:udb_su}, \eqref{eq_a:udb_du}, \eqref{eq_a:udb_sd},
we complete the proof of this theorem.
\qed

\subsection{Proof of Theorem~\ref{theo:est_sdu}}
Let $R_{\rm SDU}^{\boldsymbol{\gamma}}(f) \defeq \gamma_1 R_\SU(f) + \gamma_2 R_\DU(f) + \gamma_3 R_\SD(f)$.
We can rewrite this risk as follows.
\begin{equation}
\begin{split}
    R_{\rm SDU}^{\boldsymbol{\gamma}}(f)
    & = \frac{\pi_\S}{\pi_+ - \pi_-} \E_{X \sim \widetilde{p}_\S(\boldsymbol{x})}
    \left[ (\gamma_1 + \gamma_3 \pi_+)\ell(f(X), + 1) - (\gamma_1 + \gamma_3 \pi_-) \ell(f(X), -1)\right] \\
    & +\frac{\pi_\D}{\pi_+ - \pi_-} \E_{X \sim \widetilde{p}_\S(\boldsymbol{x})}
    \left[ -(\gamma_2 + \gamma_3 \pi_-)\ell(f(X), + 1) + (\gamma_2 + \gamma_3 \pi_+) \ell(f(X), -1)\right] \\
    & + \frac{1}{\pi_+ - \pi_-} \E_{X \sim p_\U(\boldsymbol{x})}
    \left[ (\gamma_2 \pi_+ - \gamma_1 \pi_-)\ell(f(X), + 1) + (\gamma_1 \pi_+ - \gamma_2 \pi_-) \ell(f(X), -1)\right].
\end{split}
\end{equation}
Applying the uniform deviation bounds for each in the same way with Theorem~\ref{theo:est}, this theorem can be proven.
\qed
\section{Derivation of Optimization with Double-Hinge Loss}
\label{sec_a:opt}

Suppose we use the double-hinge loss 
$\ell_{\rm DH}(z,t)=\max(-tz, \max(0, \frac{1}{2}- \frac{1}{2}tz))$.
In that case, we can rewrite the objective function in Eq.\,\eqref{eq:obj} can be represented as
\begin{equation}
\label{eq_a:sdu_obj}
\begin{split}
    \widehat{J}^{\boldsymbol{\gamma}}(\w) 
    & = \frac{\lambda}{2} \w^\top \w -\frac{\gamma_1 \pi_\S}{2 n_\S (\pi_+ - \pi_-)} \sum_{i=1}^{2 n_\S} \w^\top \phi(\widetilde{\bx}_{\S,i})
    +\frac{\gamma_2 \pi_\D}{2 n_\D (\pi_+ - \pi_-)} \sum_{i=1}^{2 n_\D} \w^\top \phi(\widetilde{\bx}_{\D,i}) \\
    & + \frac{\gamma_3 \pi_\S}{2 n_\S (\pi_+ - \pi_-)}  
    \sum_{i=1}^{2 n_\S} \left( \pi_+  \ell_{\rm DH}(\w^\top \phi(\widetilde{\bx}_{\S,i}), + 1) - \pi_-  \ell_{\rm DH}(\w^\top \phi(\widetilde{\bx}_{\S,i}), -1) \right) \\
    & - \frac{\gamma_3 \pi_\D}{2 n_\D (\pi_+ - \pi_-)}  
    \sum_{i=1}^{2 n_\D} \left( \pi_-  \ell_{\rm DH}(\w^\top \phi(\widetilde{\bx}_{\D,i}), + 1) - \pi_+  \ell_{\rm DH}(\w^\top \phi(\widetilde{\bx}_{\D,i}), -1) \right) \\
    & + \frac{1}{n_\U (\pi_+ - \pi_-)}  
    \sum_{i=1}^{n_\U} \left( (\gamma_2 \pi_+ - \gamma_1 \pi_-)  \ell_{\rm DH}(\w^\top \phi(\widetilde{\bx}_{\U,i}), + 1) + (\gamma_1 \pi_+ - \gamma_2 \pi_-)  \ell_{\rm DH}(\w^\top \phi(\widetilde{\bx}_{\U,i}), -1) \right).
\end{split}
\end{equation}
Using slack variables
$\boldsymbol{\xi}=\{\boldsymbol{\xi}_\S,\boldsymbol{\xi}_\D,\boldsymbol{\xi}_\U \}$ and
$\boldsymbol{\eta}=\{\boldsymbol{\eta}_\S,\boldsymbol{\eta}_\D,\boldsymbol{\eta}_\U \}$,
we can rewrite the optimization problem in Eq.\,\eqref{eq_a:sdu_obj} as follows:

\begin{equation}
\begin{split}
    \min_{\w, \boldsymbol{\xi}, \boldsymbol{\eta}} 
    & -\frac{\gamma_1 \pi_\S}{2 n_\S (\pi_+ - \pi_-)} \mathbf{1}^\top X_\S \w
    +\frac{\gamma_2 \pi_\D}{2 n_\D (\pi_+ - \pi_-)} \mathbf{1}^\top X_\D \w \\
    & + \frac{\gamma_3 \pi_+ \pi_\S}{2 n_\S (\pi_+ - \pi_-)} \mathbf{1}^\top \boldsymbol{\xi}_\S
      - \frac{\gamma_3 \pi_- \pi_\S}{2 n_\S (\pi_+ - \pi_-)} \mathbf{1}^\top \boldsymbol{\eta}_\S \\
    & - \frac{\gamma_3 \pi_- \pi_\D}{2 n_\D (\pi_+ - \pi_-)} \mathbf{1}^\top \boldsymbol{\xi}_\D
      + \frac{\gamma_3 \pi_+ \pi_\D}{2 n_\D (\pi_+ - \pi_-)} \mathbf{1}^\top \boldsymbol{\eta}_\D \\
    & + \frac{- \gamma_1 \pi_- + \gamma_2 \pi_+}{n_\U (\pi_+ - \pi_-)} \mathbf{1}^\top \boldsymbol{\xi}_\U
      + \frac{\gamma_1 \pi_+ - \gamma_2 \pi_-}{n_\U (\pi_+ - \pi_-)} \mathbf{1}^\top \boldsymbol{\eta}_\U \\
    & + \frac{1}{2} \w^\top \w
\end{split}
\end{equation}
\begin{equation}
\begin{split}
    \text{s.t.}~~~
    & \boldsymbol{\xi}_\S \geq 0,~\boldsymbol{\xi}_\S \geq \frac{1}{2} - \frac{1}{2}X_\S \w,~\boldsymbol{\xi}_\S \geq - X_\S \w, \\
    & \boldsymbol{\eta}_\S \geq 0,~\boldsymbol{\eta}_\S \geq \frac{1}{2} + \frac{1}{2}X_\S \w,~\boldsymbol{\eta}_\S \geq X_\S \w, \\
    & \boldsymbol{\xi}_\D \geq 0,~\boldsymbol{\xi}_\D \geq \frac{1}{2} - \frac{1}{2}X_\D \w,~\boldsymbol{\xi}_\D \geq - X_\D \w, \\
    & \boldsymbol{\eta}_\D \geq 0,~\boldsymbol{\eta}_\D \geq \frac{1}{2} + \frac{1}{2}X_\D \w,~\boldsymbol{\eta}_\D \geq X_\D \w, \\
    & \boldsymbol{\xi}_\U \geq 0,~\boldsymbol{\xi}_\U \geq \frac{1}{2} - \frac{1}{2}X_\U \w,~\boldsymbol{\xi}_\U \geq - X_\U \w, \\
    & \boldsymbol{\eta}_\U \geq 0,~\boldsymbol{\eta}_\U \geq \frac{1}{2} + \frac{1}{2}X_\U \w,~\boldsymbol{\eta}_\U \geq X_\U \w,
\end{split}
\end{equation}

where $\geq$ for vectors indicates element-wise inequality.
We can solve this optimization problem by the quadratic programming.
\section{Magnified Versions of Experimental Results}
\label{sec_a:exp}

In this section, we show magnified versions of the experimental results in Sec.\,\ref{sec:experiments}.

\begin{figure}[bthp] 
\subfloat[adult]{\includegraphics[width=.33\columnwidth]{img/sudusd_comp/adult_sudusd.pdf}}
\subfloat[banana]{\includegraphics[width=.33\columnwidth]{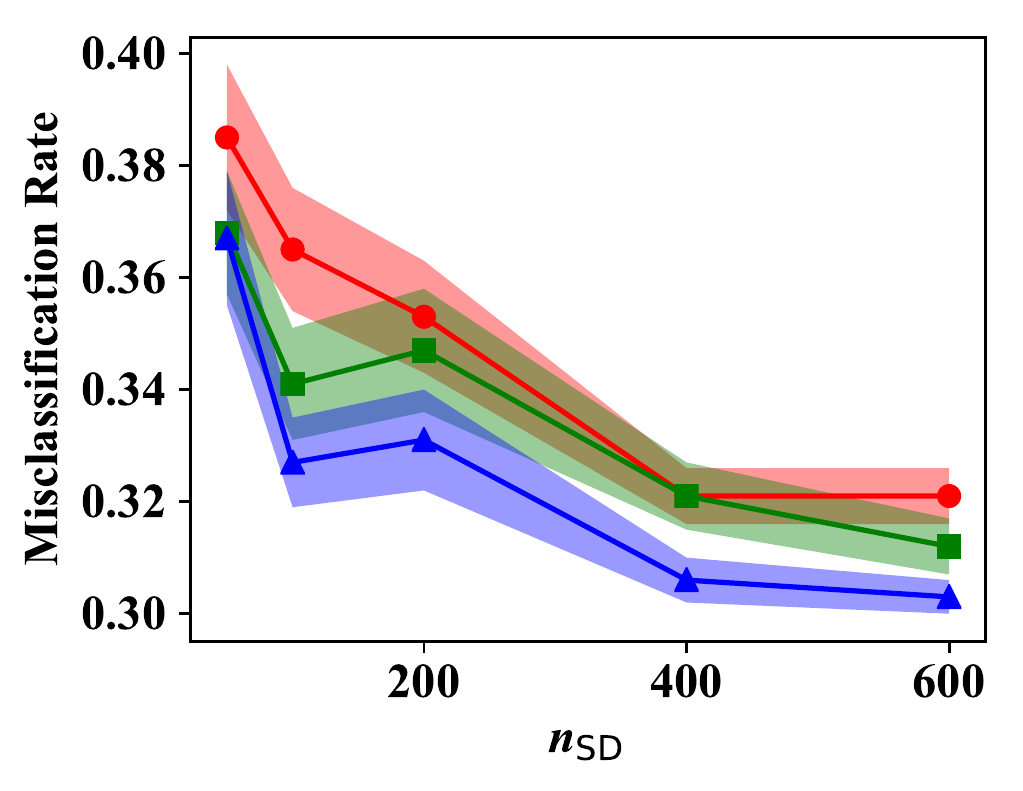}} 
\subfloat[codrna]{\includegraphics[width=.33\columnwidth]{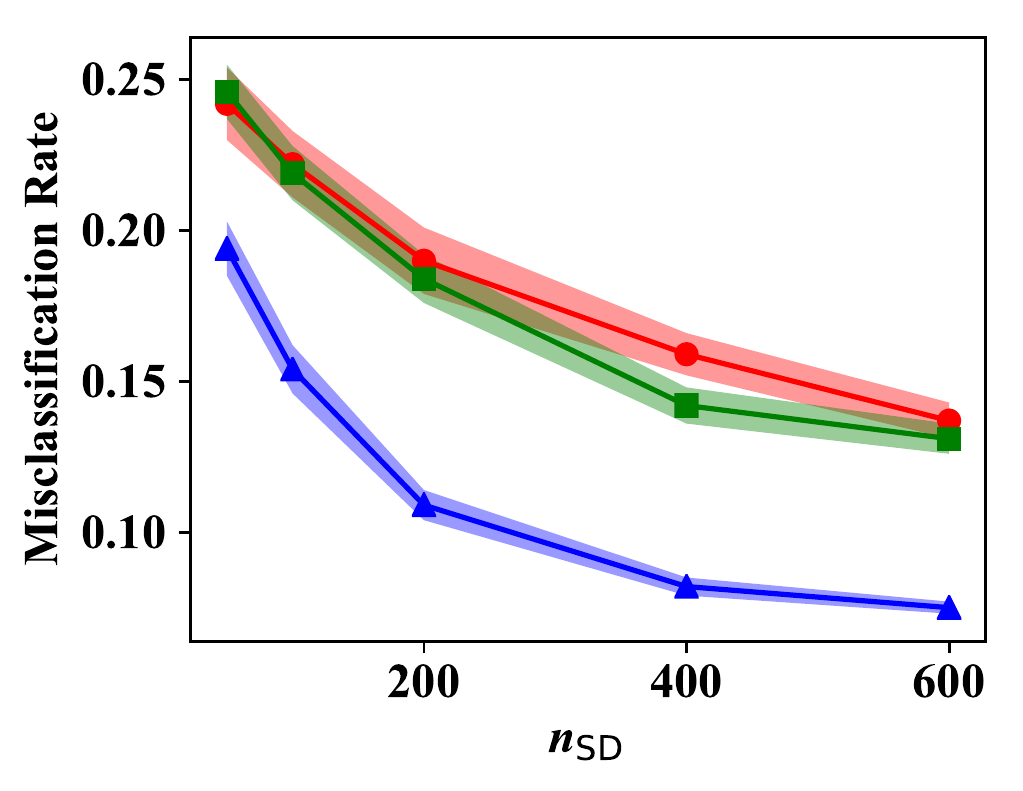}} \\ 
\subfloat[ijcnn1]{\includegraphics[width=.33\columnwidth]{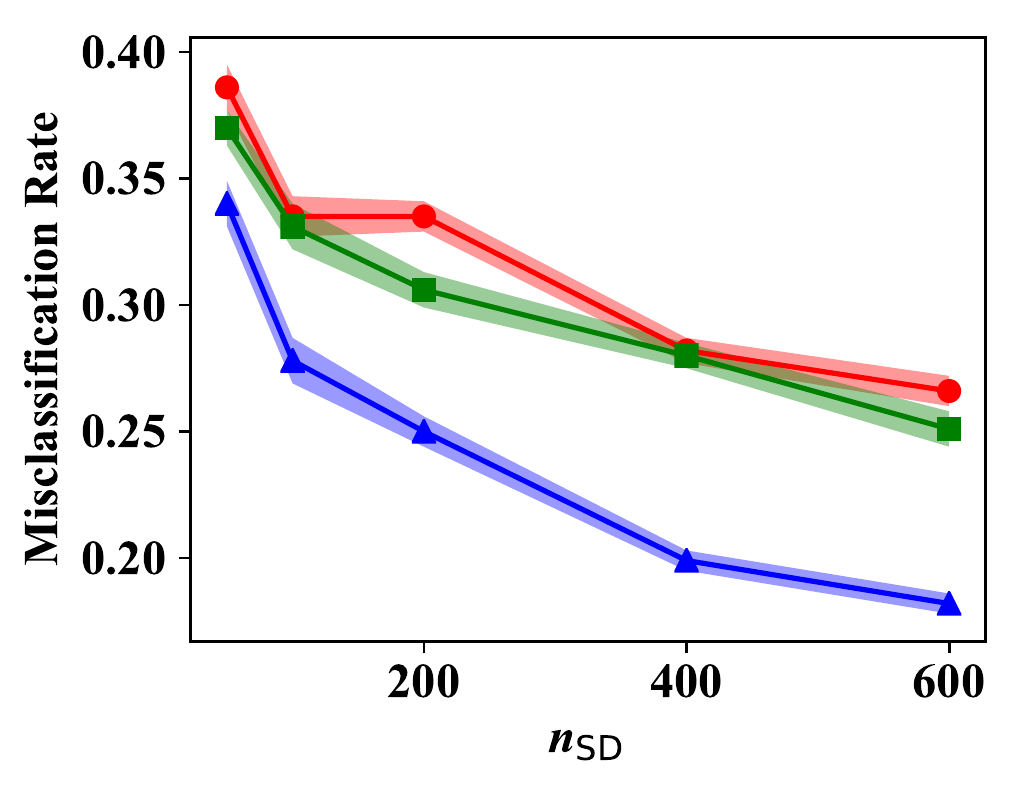}} 
\subfloat[magic]{\includegraphics[width=.33\columnwidth]{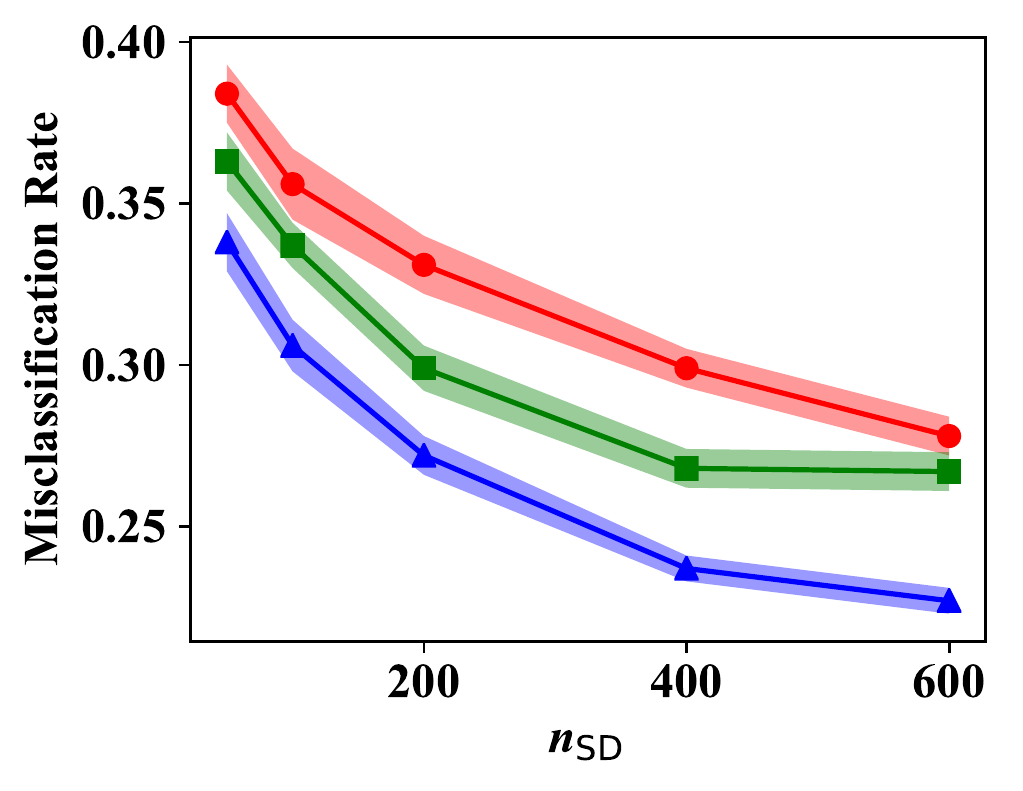}} 
\subfloat[phishing]{\includegraphics[width=.33\columnwidth]{img/sudusd_comp/phishing_sudusd.pdf}} \\  
\subfloat[phoneme]{\includegraphics[width=.33\columnwidth]{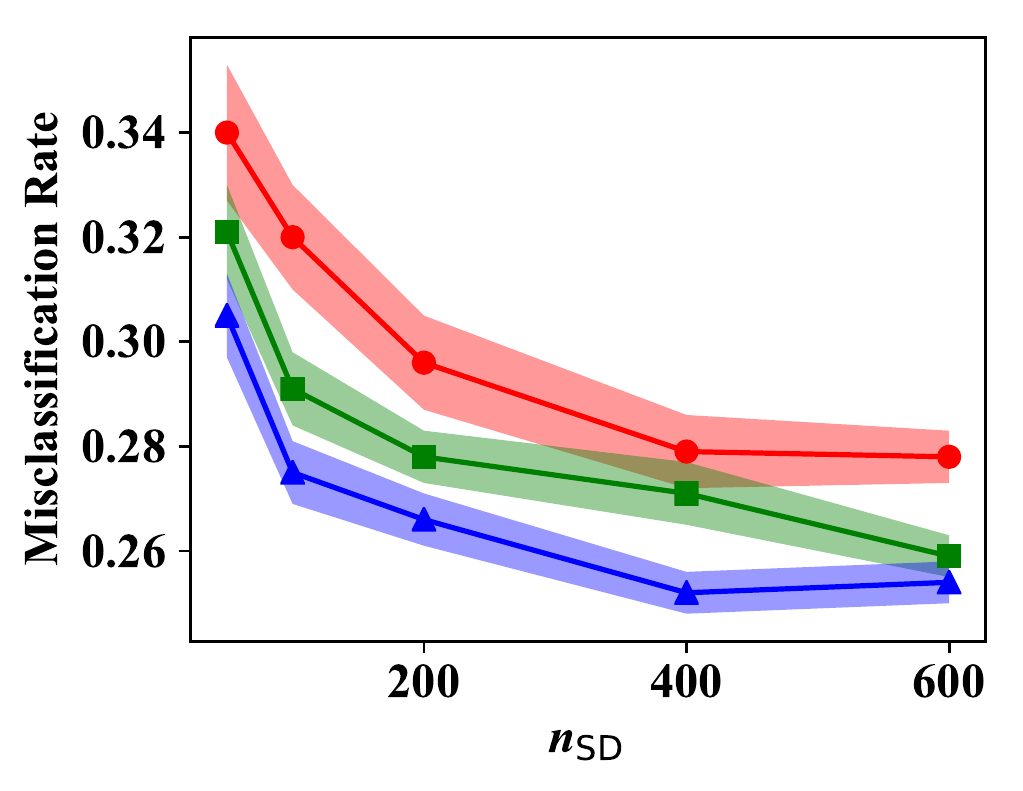}} 
\subfloat[spambase]{\includegraphics[width=.33\columnwidth]{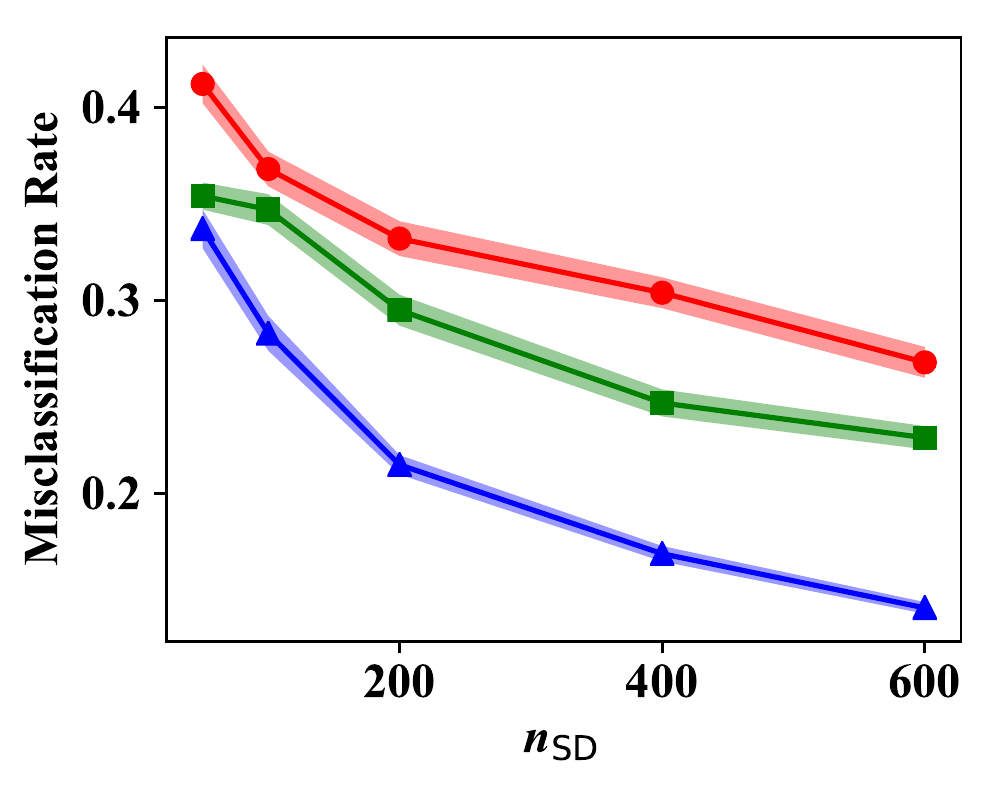}} 
\subfloat[w8a]{\includegraphics[width=.33\columnwidth]{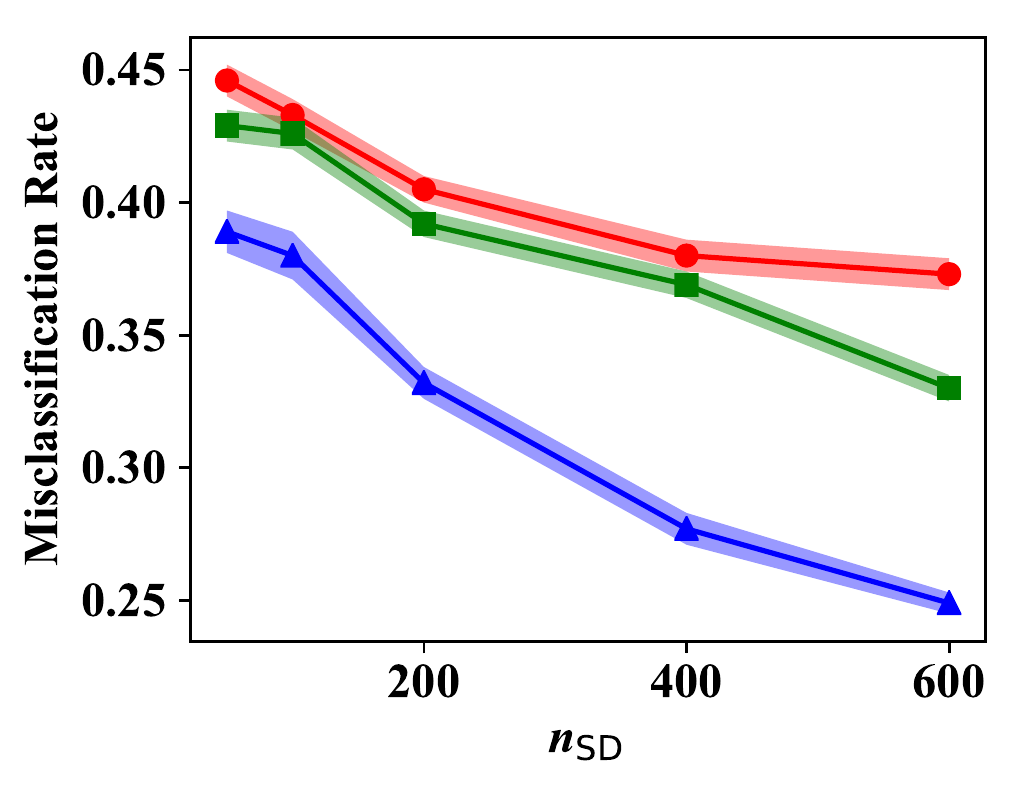}} \\ 
\subfloat[waveform]{\includegraphics[width=.33\columnwidth]{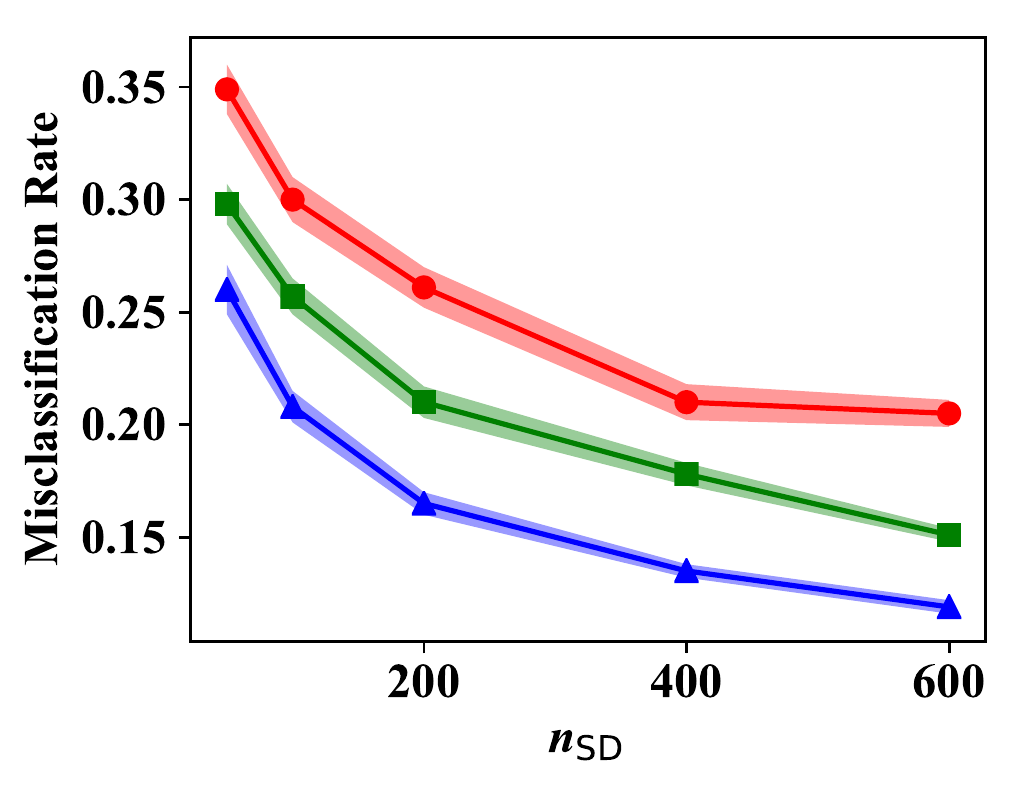}}
 \caption{
 Average misclassification rate and standard error as a function of the number of similar and dissimilar pairs over 50 trials.
 For all experiments, class prior $\pi_+$ is set to 0.7 and $n_\U$ is set to 500.
 }
\end{figure}

\begin{figure}[bthp]
\subfloat[adult]{\includegraphics[width=.33\columnwidth]{img/sdsdu_comp/adult_sdsdu.pdf}}
\subfloat[banana]{\includegraphics[width=.33\columnwidth]{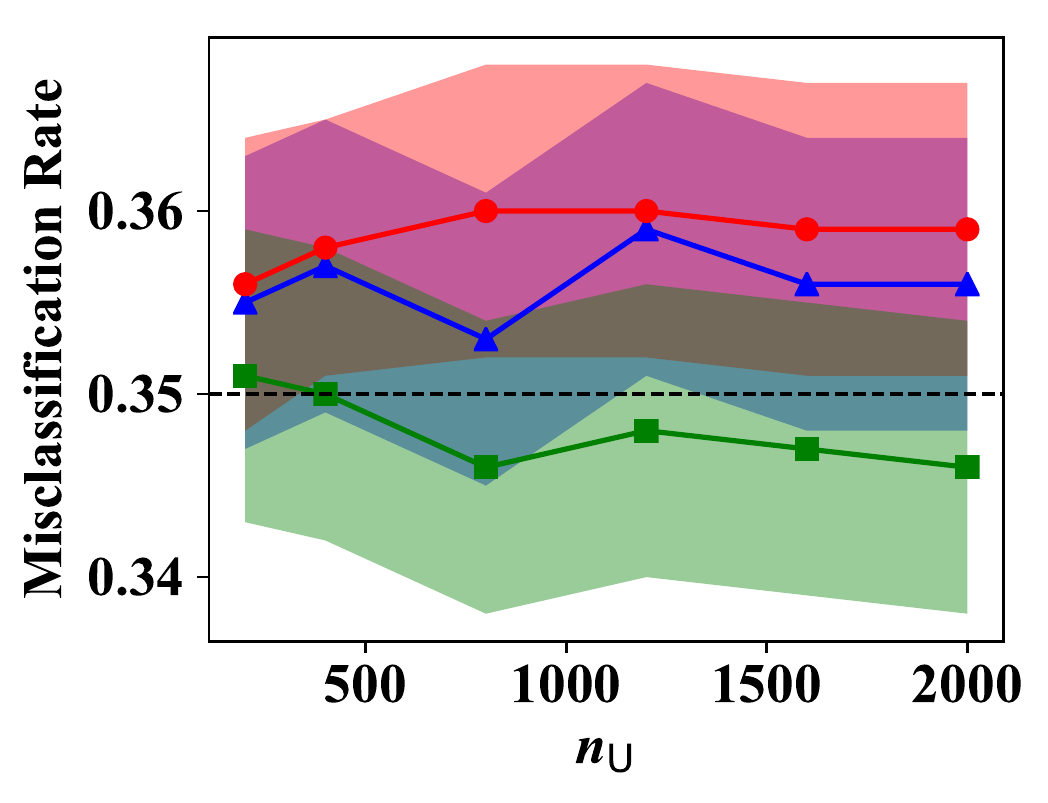}} 
\subfloat[codrna]{\includegraphics[width=.33\columnwidth]{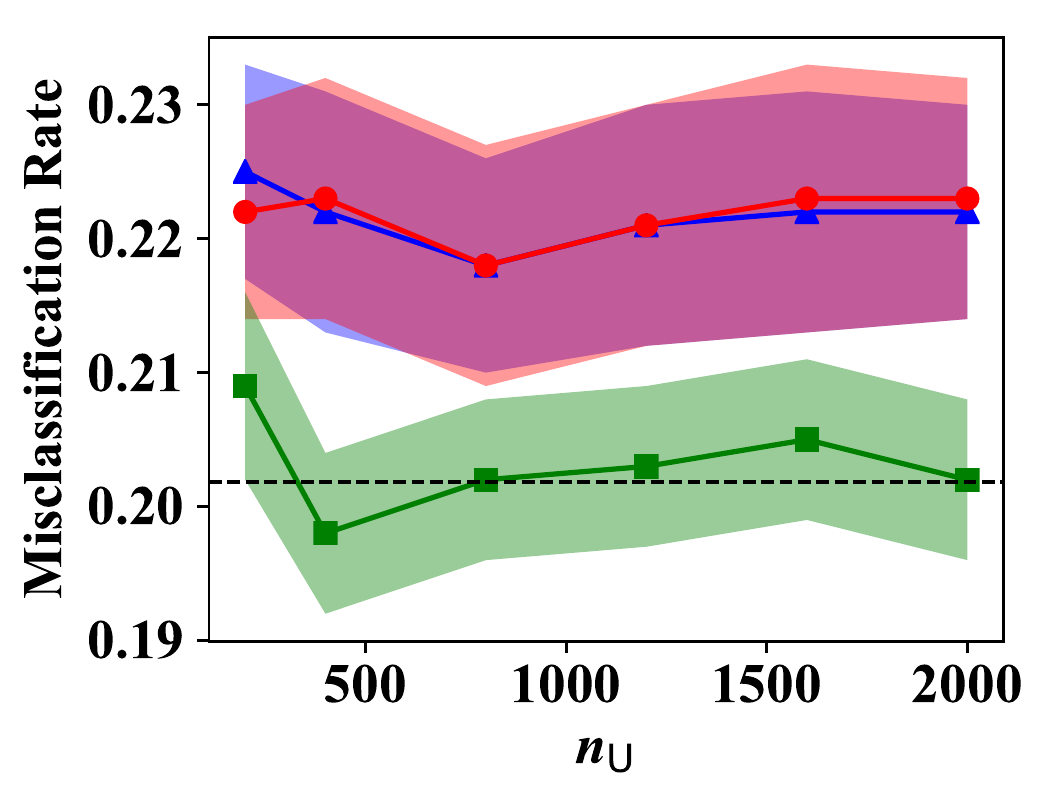}} \\
\subfloat[ijcnn1]{\includegraphics[width=.33\columnwidth]{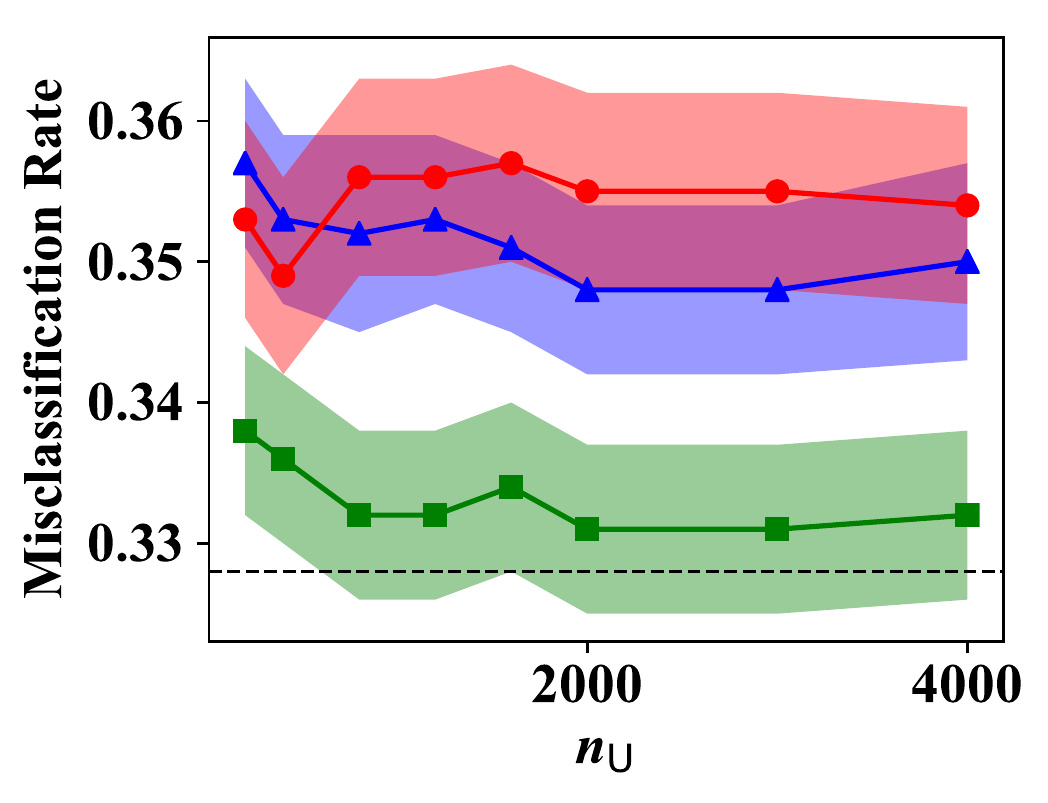}} 
\subfloat[magic]{\includegraphics[width=.33\columnwidth]{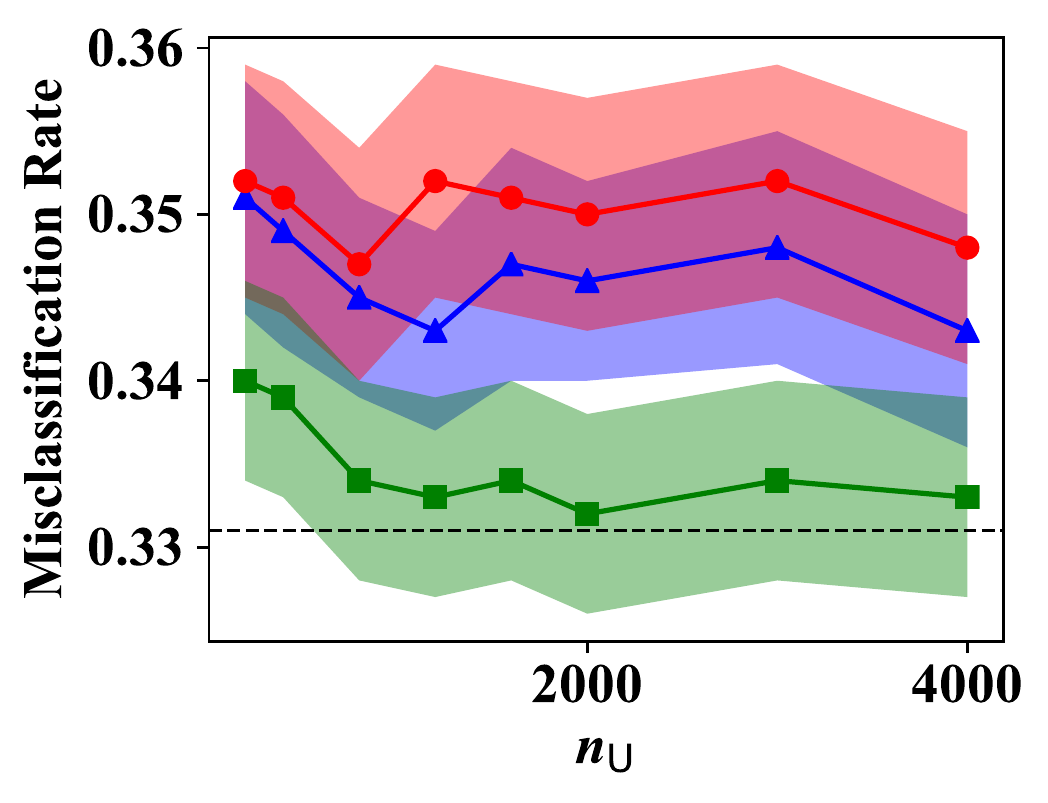}} 
\subfloat[phishing]{\includegraphics[width=.33\columnwidth]{img/sdsdu_comp/phishing_sdsdu.pdf}} \\ 
\subfloat[phoneme]{\includegraphics[width=.33\columnwidth]{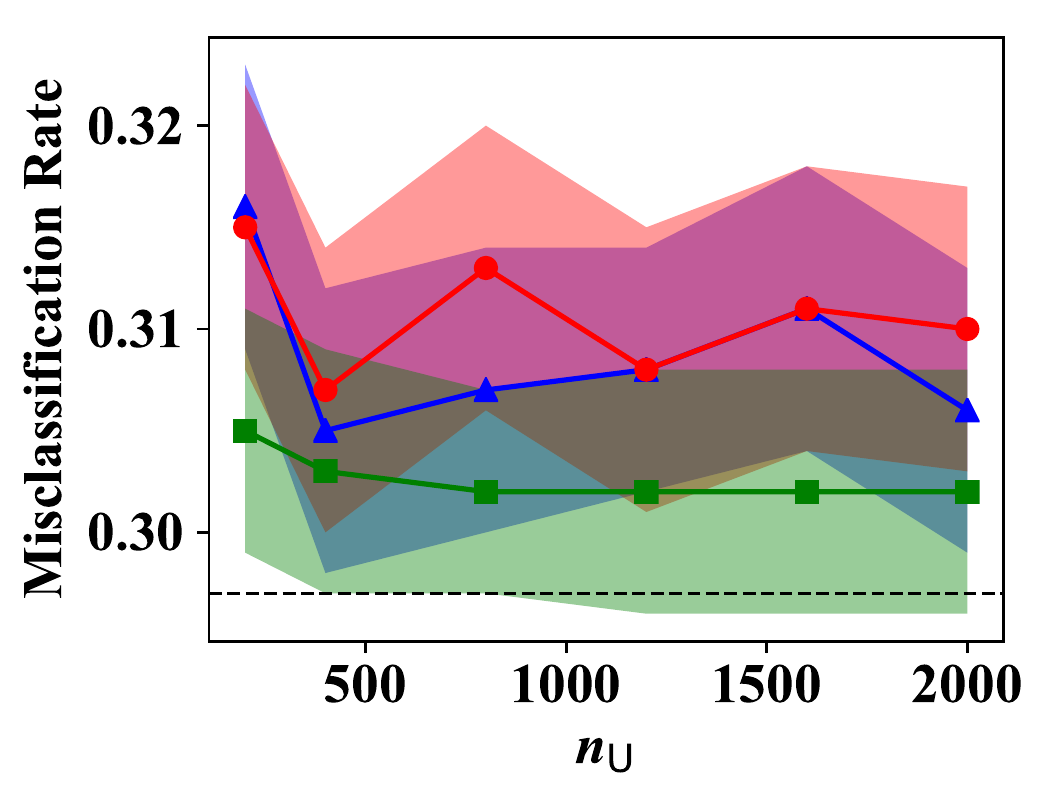}} 
\subfloat[spambase]{\includegraphics[width=.33\columnwidth]{img/sdsdu_comp/spambase_sdsdu.pdf}} 
\subfloat[w8a]{\includegraphics[width=.33\columnwidth]{img/sdsdu_comp/w8a_sdsdu.pdf}} \\
\subfloat[waveform]{\includegraphics[width=.33\columnwidth]{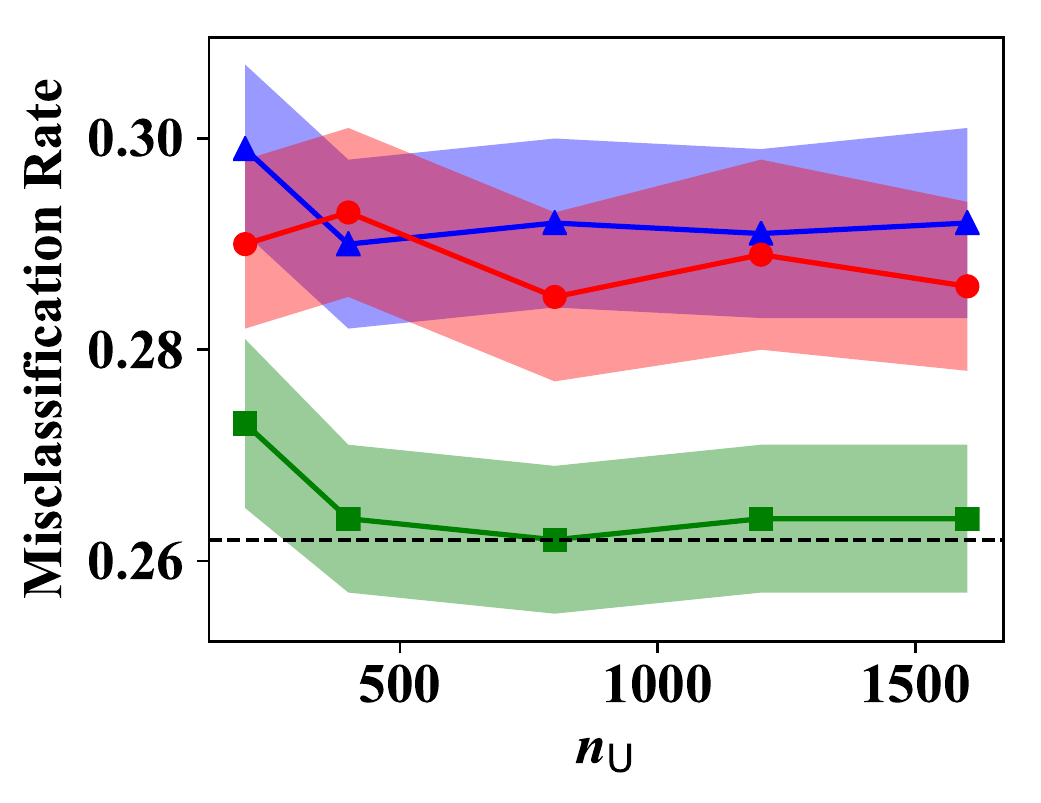}}
 \caption{
 Average misclassification rate and standard error as a function of the number of unlabeled samples over 100 trials.
 For all experiments, class prior $\pi_+$ is set to 0.7 and $n_\SD$ is set to 50.
 The mean error rate of SD classification is drawn with a black dashed line.
 }
\end{figure}

\begin{table*}[thb]
\centering
\caption{
Mean accuracy and standard error on different benchmark datasets over 50 trials.
For all experiments, class prior $\pi_+$ is set to 0.7 and $n_\U$ is set to 500.
In SDU classification, we estimate class prior from $n_\S$ and $n_\D$.
For clustering algorithms, the performances are evaluated by clustering accuracy $1-\min(r, 1-r)$, where $r$ is error rate.
Bold-faces indicate outperforming methods, chosen by one-sided t-test with the significance level 5\%.
}
\scalebox{0.66}{
\begin{tabular}{@{\extracolsep{2pt}} cccccccccccc}
\hline
 &   &  \multicolumn{2}{c}{SDSU (proposed)} & \multicolumn{2}{c}{SDDU (proposed)} & \multicolumn{2}{c}{SDSU (proposed)}   &           \multicolumn{4}{c}{Baselines}    \\
  \cline{3-4} \cline{5-6} \cline{7-8} \cline{9-12}
  Dataset &  $n_\SD$ &         Squared &     Double-Hinge &        Squared &     Double-Hinge &       Squared &     Double-Hinge &          KM &         CKM &         SSP &        ITML \\
\hline
    adult &   50 &  60.5 (0.9) &  \bf 76.1 (0.9) &  61.9 (0.9) & \bf 77.7 (0.6) &  60.5 (0.8) &  74.6 (0.9) &  65.0 (0.8) &  66.6 (1.1) &  69.5 (0.3) &  62.4 (0.7) \\
    $d=123$ &  200 &  69.9 (0.7) & \bf 82.0 (0.3) &  71.4 (0.7) & \bf 82.5 (0.3) &  69.4 (0.7) &  81.4 (0.4) &  63.3 (0.8) &  71.9 (0.9) &  69.3 (0.3) &  60.8 (0.7) \\\hline
   banana &   50 & \bf 62.1 (1.4) & \bf 61.4 (1.4) & \bf 63.9 (1.2) & \bf 63.5 (1.1) & \bf 62.1 (1.3) & \bf 62.8 (1.3) &  52.9 (0.4) &  52.7 (0.4) &  58.7 (0.7) &  53.0 (0.4) \\
   $d=2$ &  200 & \bf 65.3 (1.0) & \bf 65.6 (1.0) & \bf 66.5 (0.8) & \bf 66.9 (0.7) & \bf 65.6 (0.8) & \bf 66.7 (0.8) &  52.5 (0.2) &  52.5 (0.2) & \bf 66.5 (1.3) &  52.5 (0.2) \\\hline
   codrna &   50 & \bf 77.5 (1.3) &  67.5 (1.0) & \bf 78.1 (1.1) &  68.5 (0.8) & \bf 78.0 (1.1) &  72.8 (1.2) &  62.6 (0.5) &  61.5 (0.4) &  54.6 (1.0) &  62.7 (0.5) \\
   $d=8$ &  200 & \bf 87.1 (0.7) &  72.2 (0.6) & \bf 87.7 (0.6) &  72.7 (0.7) & \bf 87.3 (0.6) &  84.3 (0.9) &  62.8 (0.5) &  59.5 (0.5) &  53.2 (0.7) &  62.5 (0.5) \\\hline
   ijcnn1 &   50 &  65.0 (0.9) & \bf 68.0 (1.0) &  64.7 (0.8) & \bf 68.8 (0.9) &  64.4 (0.8) & \bf 66.9 (0.9) &  55.5 (0.6) &  54.7 (0.5) &  60.9 (0.8) &  55.8 (0.6) \\
   $d=22$ &  200 &  74.5 (0.7) & \bf 76.2 (0.6) & \bf 75.1 (0.7) & \bf 76.3 (0.5) &  73.8 (0.7) & \bf 76.1 (0.5) &  54.2 (0.3) &  53.1 (0.3) &  59.5 (0.8) &  54.3 (0.4) \\\hline
    magic &   50 & \bf 63.4 (1.0) &  62.8 (1.1) & \bf 65.5 (0.9) & \bf 65.1 (1.0) & \bf 64.0 (0.9) & \bf 63.6 (1.1) &  52.4 (0.2) &  51.8 (0.2) &  52.7 (0.3) &  52.5 (0.2) \\
    $d=10$ &  200 & \bf 72.0 (0.6) &  71.2 (0.8) & \bf 73.0 (0.6) &  71.4 (0.7) & \bf 71.8 (0.6) & \bf 72.5 (0.5) &  52.0 (0.2) &  51.7 (0.2) &  52.6 (0.3) &  52.0 (0.2) \\\hline
 phishing &   50 &  67.2 (1.0) &  76.7 (1.3) &  69.4 (0.8) & \bf 80.5 (0.9) &  67.1 (1.0) &  77.6 (1.1) &  62.6 (0.3) &  62.6 (0.3) &  68.1 (0.3) &  62.5 (0.3) \\
 $d=68$ &  200 &  81.6 (0.7) & \bf 86.7 (0.6) &  81.7 (0.7) & \bf 87.0 (0.4) &  79.3 (0.6) &  85.4 (0.5) &  62.6 (0.3) &  62.8 (0.3) &  68.4 (0.3) &  62.6 (0.3) \\\hline
  phoneme &   50 &  66.9 (1.2) & \bf 67.3 (1.4) &  67.9 (0.9) & \bf 69.2 (0.9) &  67.0 (1.2) &  67.2 (1.4) &  \bf 67.8 (0.3) & \bf 68.9 (0.5) &  66.5 (1.0) & \bf 67.8 (0.3) \\
  $d=5$ &  200 &  72.5 (0.6) & \bf 73.4 (0.6) &  73.5 (0.5) & \bf 74.4 (0.4) &  73.2 (0.5) &  74.2 (0.5) &  67.8 (0.3) &  71.0 (0.6) &  72.0 (0.7) &  67.9 (0.3) \\\hline
 spambase &   50 &  65.3 (1.0) & \bf 82.5 (0.7) &  66.7 (0.8) & \bf 82.9 (0.6) &  65.2 (0.8) &  80.3 (0.8) &  63.7 (1.1) &  64.2 (1.1) &  70.4 (0.3) &  61.4 (1.1) \\
 $d=57$ &  200 &  75.9 (0.9) & \bf 86.9 (0.3) &  77.9 (0.7) & \bf 87.5 (0.3) &  74.5 (0.8) &  86.7 (0.3) &  61.8 (1.2) &  70.4 (0.6) &  70.7 (0.3) &  60.6 (1.2) \\\hline
      w8a &   50 &  57.2 (0.9) & \bf 70.9 (1.2) &  60.8 (0.9) & \bf 73.2 (0.9) &  55.9 (0.7) &  70.2 (1.3) &  69.3 (0.3) &  66.0 (0.7) &  64.2 (0.6) &  69.0 (0.3) \\
      $d=300$ &  200 &  62.6 (0.6) & \bf 80.1 (0.7) &  64.1 (0.7) & \bf 80.2 (0.7) &  61.0 (0.8) &  77.8 (0.6) &  69.3 (0.3) &  56.3 (0.6) &  67.2 (0.7) &  68.6 (0.3) \\\hline
 waveform &   50 &  69.9 (1.2) & \bf 80.8 (1.3) &  72.9 (1.0) & \bf 83.2 (1.0) &  71.8 (0.9) &  79.6 (1.3) &  51.5 (0.2) &  51.6 (0.2) &  53.3 (0.3) &  51.5 (0.2) \\
 $d=21$ &  200 &  82.2 (0.7) & \bf 86.3 (0.6) &  83.2 (0.6) & \bf 86.7 (0.6) &  82.0 (0.7) & \bf 86.3 (0.5) &  51.5 (0.2) &  51.5 (0.1) &  53.1 (0.3) &  51.5 (0.2) \\
\hline
\end{tabular}}
\end{table*}
\end{document}